\pdfoutput=1

\documentclass[letterpaper]{article}
\usepackage{uai2020}
\usepackage[margin=1in]{geometry}

\usepackage{times}

\usepackage{hyperref}       
\usepackage{url}            
\usepackage{booktabs}       
\usepackage{amsfonts}       
\usepackage{nicefrac}       
\usepackage{microtype}      
\usepackage{amsmath,amsthm}
\usepackage{mathtools}
\usepackage[dvipsnames]{xcolor}
\usepackage[linesnumbered,ruled,vlined]{algorithm2e}
\usepackage{amssymb}
\usepackage{bbm}
\usepackage{enumitem}
\usepackage{algorithm2e}
\usepackage{dirtytalk}
\usepackage{pifont}
\usepackage{tikz}
\usetikzlibrary{arrows.meta}
\usepackage{float}
\usepackage{cleveref}
\usepackage{hyperref}
\usepackage{soul}
\usepackage{dsfont}

\DeclareMathOperator*{\argmin}{arg\,min}

\usepackage{bm} 

\usepackage[authoryear, round, sort&compress]{natbib}

\newcommand{\SPIN}{{\small\textsc{SPIN}}\xspace}
\newcommand{\LEFT}{{\small\textsc{LEFT}}\xspace}
\newcommand{\RIGHT}{{\small\textsc{RIGHT}}\xspace}
\newcommand{\SELFLOOP}{{\small\textsc{SELF-LOOP}}\xspace}
\newcommand{\CENTER}{{\small\textsc{CENTER}}\xspace}
\newcommand{\NOISY}{{\small\textsc{NOISY}}\xspace}

\newcommand{\MODEST}{{\small\textsc{ModEst}}\xspace}
\newcommand{\MODESTtitle}{{\textsc{ModEst}}\xspace}
\newcommand{\FWMODEST}{{\small\textsc{FW-ModEst}}\xspace}
\newcommand{\FW}{{\small\textsc{FW}}\xspace}

\newcommand{\MaxEnt}{{\small\textsc{MaxEnt}}\xspace}
\newcommand{\MaxEnttitle}{{\textsc{MaxEnt}}\xspace}

\newcommand{\WeightedMaxEnt}{{\small\textsc{Weighted-MaxEnt}}\xspace}
\newcommand{\WMaxEnt}{{\small\textsc{WMaxEnt}}\xspace}

\newcommand{\LSE}{{\textsc{LSE}}}

\newcommand{\EVI}{{\small\textsc{EVI}}\xspace}

\newcommand{\TOCUCRLtwo}{{\small\textsc{TOC-UCRL2}}\xspace}

\newcommand\myineeqa{\mathrel{\stackrel{\makebox[0pt]{\mbox{\normalfont\tiny (a)}}}{\leq}}}
\newcommand\myineeqb{\mathrel{\stackrel{\makebox[0pt]{\mbox{\normalfont\tiny (b)}}}{\leq}}}
\newcommand\myineeqc{\mathrel{\stackrel{\makebox[0pt]{\mbox{\normalfont\tiny (c)}}}{\leq}}}

\newcommand{\A}{\mathcal{A}}
\newcommand{\calS}{\mathcal{S}}
\newcommand{\SA}{\mathcal{S} \times \mathcal{A}}
\newcommand{\SAS}{\mathcal{S} \times \mathcal{A} \times \mathcal{S}}
\newcommand{\calE}{\mathcal{E}}
\newcommand{\calW}{\mathcal{W}}
\newcommand{\calL}{\mathcal{L}}
\newcommand{\calB}{\mathcal{B}}

\newcommand{\calG}{\mathcal{G}}

\makeatletter
\newcommand\footnoteref[1]{\protected@xdef\@thefnmark{\ref{#1}}\@footnotemark}
\makeatother

\newtheorem{theorem}{Theorem}
\newtheorem{lemma}{Lemma}

\newtheorem{proposition}{Proposition}

\newtheorem{definition}{Definition}
\newtheorem{assumption}{Assumption}

\newcommand{\wt}[1]{\widetilde{#1}}
\newcommand{\wb}[1]{\overline{#1}}

\newcommand{\wh}[1]{\widehat{#1}}

\DeclarePairedDelimiter\abs{\lvert}{\rvert}%
\DeclarePairedDelimiter\norm{\lVert}{\rVert}%

\let\originalleft\left
\let\originalright\right
\renewcommand{\left}{\mathopen{}\mathclose\bgroup\originalleft}
\renewcommand{\right}{\aftergroup\egroup\originalright}

\makeatletter
\renewcommand*{\@fnsymbol}[1]{\ensuremath{\ifcase#1\or 1\or 2\or 3\or
    \mathsection\or \mathparagraph\or \|\or **\or \dagger\dagger
    \or \ddagger\ddagger \else\@ctrerr\fi}}
\makeatother

\title{Active Model Estimation in Markov Decision Processes}

\author{
    {\bf Jean Tarbouriech}\\
    Facebook AI Research \& Inria Lille, SequeL team
    \And
    {\bf Shubhanshu Shekhar}\\
    University of California, San Diego
    \AND
    {\bf Matteo Pirotta}\\
    Facebook AI Research
    \And
    {\bf Mohammad Ghavamzadeh}\\
    Facebook AI Research
    \And
    {\bf Alessandro Lazaric}\\
    Facebook AI Research
}

\begin{document}

\maketitle

\begin{abstract}
We study the problem of efficient exploration in order to learn an accurate model of an environment, modeled as a Markov decision process (MDP). Efficient exploration in this problem requires the agent to identify the regions in which estimating the model is more difficult and then exploit this knowledge to collect more samples there. In this paper, we formalize this problem, introduce the first algorithm to learn an $\epsilon$-accurate estimate of the dynamics, and provide its sample complexity analysis. While this algorithm enjoys strong guarantees in the large-sample regime, it tends to have a poor performance in early stages of exploration. To address this issue, we propose an algorithm that is based on maximum weighted entropy, a heuristic that stems from common sense and our theoretical analysis. The main idea here is to cover the entire state-action space with the weight proportional to the noise in the transitions. Using a number of simple domains with heterogeneous noise in their transitions, we show that our heuristic-based algorithm outperforms both our original algorithm and the maximum entropy algorithm in the small sample regime, while achieving similar asymptotic performance as that of the original algorithm.
\end{abstract}


\vspace{-0.1in}
\section{INTRODUCTION}
\vspace{-0.1in}

In most decision problems, the agent is provided with a goal that it tries to achieve by maximizing a reward signal. In such problems, the agent explores the environment in order to identify the high reward situations and reach the goal faster and more efficiently. Although solving a problem (achieving a goal) is usually the ultimate objective, it is sometimes equally important for an agent to understand its environment without pursuing any goals. In such scenarios, no reward function is defined and the agent explores the state-action space in order to discover what is possible and how the environment works. We refer to this scenario as {\em reward-free} or {\em unsupervised} exploration. Several objectives have been studied in reward-free exploration, including discovering incrementally reachable states~\citep{Lim12AE}, uniformly covering the state space~\citep{hazan2019provably}, estimating state-dependent random variables~\citep{tarbouriech2019active}, a broad class of objectives that are defined as functions of the state visitation frequency induced by the agent's behavior~\citep{cheung2019arvixexploration,cheung2019regret}, and learning a model of the environment that is suitable for computing near-optimal policies for a given collection of reward functions~\citep{Jin20RF}. 
A reward-free exploration algorithm is evaluated by the amount of exploration it uses to learn its objective.

None of the works above, except~\citep{Jin20RF}, focuses on learning the dynamics of an environment modeled as a Markov decision process (MDP). Even in~\citep{Jin20RF}, the setting is the simpler finite-horizon MDP, where many states are often irrelevant as they have no impact in defining the optimal policy for any reward functions.
Although the problem of learning the dynamics has not been rigorously studied in the reward-free exploration setting, \citep{araya2011active} proposed several heuristics for this problem. Studying active exploration for model estimation is also important in the theoretical understanding of the {\em simulation-to-real} problem, where the goal is to start with an inaccurate model (simulator) of the environment and learn a better one with minimum interaction with the world.

In this paper, we formalize the problem of reward-free exploration where the objective is to estimate both a uniformly accurate (minimizing the maximum error) and an average accurate (minimizing the average error) model of the environment, i.e.,~the transition probability function of the MDP. We identify the form of the model estimation error for a given policy and show that it depends on how often the policy visits noisy state-action pairs, i.e.,~those whose transition probability has high variance. Since optimizing the model estimation error over the policies is difficult, we upper bound it using Bernstein's inequality and obtain an objective function that relates the structure of the MDP with the accuracy of a model estimated by a certain state-action visitation (stationary distribution of a policy). We build on~\citep{tarbouriech2019active} and propose an algorithm that optimizes this objective over stationary distributions and prove sample complexity bounds for the accuracy (on average and in worst case) of the estimated model. In particular, our analysis highlights the intrinsic difficulty of the model estimation problem.

Our ``exact'' algorithm may be inefficient due to the specific form of the objective function we optimize and it tends to perform well only in large sample regimes (asymptotically). Furthermore, our theoretical guarantees only hold under restrictive assumptions on the MDP (i.e., ergodicity). To alleviate these limitations, we replace the objective function used by our algorithm with maximizing {\em weighted entropy}, where the goal is to visit state-action pairs weighted by the noise in their transitions. Although this is a heuristic, it stems from our derived objective function and is more tailored to the model estimation problem than the popular maximum entropy objective~\citep{hazan2019provably,cheung2019arvixexploration} that uniformly covers the state space. We derive an algorithm based on the maximum weighted entropy heuristic, by modifying and extending the algorithm in~\citep{cheung2019arvixexploration}, and we prove regret guarantees w.r.t.\ the stationary distribution maximizing the weighted entropy. 

The main contributions of this paper are: {\bf 1)} We introduce an objective function and use it to derive and analyze an algorithm for accurate MDP model estimation in the reward-free exploration setting. The algorithm extends the one in~\citep{tarbouriech2019active} by removing the requirement of knowing the model.
{\bf 2)} When we switch from our original objective function to the maximum weighted entropy heuristic, we show how the algorithm in~\citep{cheung2019arvixexploration} can be modified to handle an unknown objective (since the weights, i.e.,~noise in transitions are not known in advance).\footnote{Interestingly, the algorithm in~\citep{cheung2019arvixexploration} \textit{cannot} be applied to the original model estimation problem, as the objective function violates the assumptions in~\citep{cheung2019arvixexploration}.} {\bf 3)} We show that the algorithm based on the maximum weighted entropy heuristic outperforms both our theoretically driven algorithm and the maximum entropy algorithm in the small sample regime, while achieving similar asymptotic performance as that of our exact algorithm. Finally, {\bf 4)} we consider an undiscounted infinite-horizon setting, which is more general than the finite horizon MDP setting considered in the recent work~\citep{Jin20RF}, where a large portion of the states can be considered as irrelevant.




\section{PROBLEM FORMULATION}
\label{sec_problem_formulation}

We consider a reward-free finite \textit{Markov decision process} (MDP) $M := \langle\mathcal{S}, \mathcal{A}, p \rangle$, where $\mathcal{S}$ is the set of $S$ states, $\mathcal{A}$ is the set of $A$ actions. Each state-action pair \mbox{$(s,a) \in \mathcal{S} \times \mathcal{A}$} is characterized by an \textit{unknown} transition probability distribution $p(\cdot \, \vert \, s,a)$ over next states. We denote by $\pi$ a stochastic non-stationary policy that maps a history $(s_0, a_0, s_1, \ldots, s_t)$ of states and actions observed up to time $t$ to a distribution over actions $\A$. We also consider stochastic \textit{stationary} policies $\pi : \calS \rightarrow\Delta(\A)$, which map the current state $s_t$ to a distribution over actions.


At a high-level, the objective of the agent is to estimate as accurately as possible the unknown transition dynamics. We refer to this objective as the \textit{model estimation problem}, or \MODEST for short. More formally, the agent interacts with the environment by following a possibly non-stationary policy $\pi$ and, after $n$ time steps, it returns an estimate $\wh{p}_{\pi,n}$ of the transition dynamics. We evaluate the accuracy of $\wh{p}_{\pi,n}$ in terms of its $\ell_1$-norm distance (i.e., total variation) w.r.t.\ the model $p$ either on average or in the worst-case over the state-action space:
\begin{align}
\mathcal{E}_{\pi,n} &:= \frac{1}{SA} \!\!\!\sum_{(s,a) \in \SA } \norm{\widehat{p}_{\pi,n}(\cdot \vert s,a) - p(\cdot \vert s,a)}_1, \label{eq:average.modest}\\
\mathcal{W}_{\pi,n} &:= \max_{(s,a) \in \SA } \norm{\widehat{p}_{\pi,n}(\cdot \vert s,a) - p(\cdot \vert s,a)}_1.\label{eq:max.modest}
\end{align}
We say that an estimate $\wh p_{\pi,n}$ is $(\epsilon,\delta)$-accurate in the sense of $\calE$ (resp.~$\calW$) if $\mathbb{P}\big( \mathcal{E}_{\pi,n} \!\leq\! \epsilon \big) \geq 1 - \delta$, i.e., it is $\epsilon$-accurate with at least $1-\delta$ probability (resp.~$\calW_{\pi,n}$). We then measure the sample complexity of a policy $\pi$ as follows.
\begin{definition}\label{def:complexity}
	Given an error $\epsilon > 0$ and a confidence level $\delta \in (0,1)$, the \MODEST \textit{sample complexity} of an agent executing a policy $\pi$ is defined as
	\begin{align*}
	\mathcal{C}^{\mathcal{E}}_{\MODEST}(\pi, \epsilon, \delta) &:= \min \Big\{n \geq 1\!: \mathbb{P}\big( \mathcal{E}_{\pi,n} \!\leq\! \epsilon \big) \geq 1 - \delta \Big\}, \\
	\mathcal{C}^{\mathcal{W}}_{\MODEST}(\pi, \epsilon, \delta) &:= \min \Big\{n \geq 1\!: \mathbb{P}\big( \mathcal{W}_{\pi,n} \!\leq\! \epsilon \big) \geq 1 - \delta \Big\}.
	\end{align*}
\end{definition}
According to the definition above, the objective is to find an algorithm that is able to return an $(\epsilon,\delta)$-accurate model with as little number of samples as possible, i.e., minimize the sample complexity. While the definition above follows a standard PAC-formulation, it is possible to consider the case when a fixed budget $n$ is provided in advance and the agent should return an estimate which is as accurate as possible with confidence $1-\delta$. In the following, we consider both settings.

\textbf{Discussion.} While both $\mathcal{E}_{\pi,n}$ and $\mathcal{W}_{\pi,n}$ effectively formalize the objective of accurate estimation of the dynamics, they have specific advantages and disadvantages.

If $\mathcal{W}_{\pi,n}$ is below $\epsilon$, it is possible to compute near-optimal policies for any reward~\citep[e.g., ][]{kearns2002near}.

\begin{proposition}\label{p:sim.lemma}
For any $\gamma\in(0,1)$ and reward function $r(s,a)\in [0,1]$, let $\wh\pi$ be the infinite-horizon $\gamma$-discounted optimal policy computed using the dynamics $\wh p$ returned after the \MODEST phase. If $\wh p$ is $(\epsilon,\delta)$-accurate in the sense of $\calW$, then $\wh\pi$ is guaranteed to be $O(\epsilon(1-\gamma)^{-2})$-optimal in the exact MDP with probability at least $1-\delta$.
\end{proposition}

While this shows the advantage of targeting the worst-case accuracy, the maximum over state-action pairs in the formulation of $\mathcal{W}_{\pi,n}$ makes it more challenging to optimize than $\calE_{\pi,n}$. Furthermore, consider the case where one single state-action pair has a very noisy dynamics but is overall uninteresting (e.g., the so-called noisy TV case), then targeting the worst-case error may require the agent to spend most of its time in estimating the dynamics of that single state-action pair, instead of collecting information about the rest of the MDP. A similar issue may rise in large MDPs, where lowering the estimation error uniformly over the whole state-action space may be prohibitive and lead to extremely large sample complexity.

Alternative metrics other than $\ell_1$ could be used to measure the distance between $\widehat{p}_{\pi,n}$ and $p$. For instance, given the categorical nature of the transition distributions, a natural choice is to use the KL-divergence. Nonetheless, it has been shown that the KL-divergence is not a well-suited metrics for active learning as it promotes strategies close to uniform sampling~\citep{shekhar2019adaptive}. Furthermore, other metrics may not provide any performance guarantees such as Prop.~\ref{p:sim.lemma}.

One may wonder why we rely on Def.~\ref{def:complexity} to measure the accuracy, instead of the closely related objective defined in~\citep{tarbouriech2019active}, where the \textit{expected} estimation error is used in place of its high-probability version. A careful inspection of the analysis in~\citep{tarbouriech2019active} shows that they critically rely on the independence between the stochastic transitions in the MDP and the randomness in the observations collected in each state. 
This argument cannot be applied to our case, where the observations of interest are the next state $s'$ generated during transitions $(s,a)$ to $s'$. This makes transitions and observations intrinsically correlated and it prevents us from using their analysis.

While minimizing the sample complexity $\mathcal{C}_{\MODEST}$ over policies $\pi$ is a well-posed objective, the dependency on $\pi$ is implicit and the space of non-stationary policies is combinatorial in states, actions and time, thus making it impossible to directly optimize for $\mathcal{C}_{\MODEST}$. In the next section we refine the error definitions to derive an explicit convex surrogate for stationary policies.


\vspace{-0.1in}
\subsection{A Surrogate Convex Objective Function}
\label{ssec:convex.surrogate}
\vspace{-0.1in}

We focus on model estimates defined as the empirical frequency of transitions. After $n$ steps, for any state-action pair $(s,a)$ and any next state $s'$, we define
\begin{align}\label{eq:emp.estimate}
\widehat{p}_{\pi, n}(s' \vert s,a) := \frac{T_{\pi, n}(s,a,s')}{T_{\pi, n}^+(s,a)},
\end{align}
where $T_{\pi,n}(s,a) := \sum_{t=1}^n \mathds{1}_{(s_t,a_t) = (s,a)}$ denotes the (random) number of times action $a$ was taken in state $s$ during the execution of policy $\pi$ over $n$ steps. Similarly, we define $T_{\pi,n}(s,a,s') := \sum_{t=1}^n \mathds{1}_{(s_t,a_t,s_{t+1}) = (s,a,s')}$ and we define $x^+ := \max \{ x, 1\}$.

We introduce the notion of ``noise'' associated to each state-action pair.

\begin{definition}
	For each state-action pair $(s,a)$, we define the \textit{transitional noise} at $(s,a)$ as
	\begin{align*}
	V(s,a) := \sum_{s' \in \mathcal{S}} \sqrt{\sigma^2_p(s'\vert s,a)}/\sqrt{S},
	\end{align*}
	where $\sigma^2_p(s'\vert s,a)$ is the variance of the transition from $(s,a)$ to $s'$, i.e., $\sigma^2_p(s'\vert s,a) := p(s' \vert s,a)(1-p(s' \vert s,a))$.
\end{definition}

Denote by $\Gamma(s,a) := \norm{p(\cdot \vert s,a)}_0$ the support of $p(\cdot \vert s,a)$, as well as $\Gamma := \max_{s,a} \Gamma(s,a)$. From the Cauchy-Schwarz inequality, we have
	\begin{align*}
	0 \leq V(s,a) \leq \sqrt{\Gamma(s,a) - 1}/\sqrt{S} \leq 1,
	\end{align*}
where the lower bound $V(s,a) = 0$ is achieved for all $(s,a) \in \SA$ when $M$ is a deterministic MDP.

We connect the error functions to a term depending on the transitional noise and the number of state-action visits.

\begin{proposition}\label{p:upper.bound}
For any fixed $n$ and any policy $\pi$, with probability at least $1-\delta/3$,
\begin{align*}
\mathcal{E}_{\pi,n} & \leq \frac{4}{SA} \log\left(\frac{6 SA n}{\delta} \right) \sum_{s,a} \mathcal{F}(s,a; T_{\pi,n}),\\
\mathcal{W}_{\pi,n} & \leq 4 \log\left(\frac{6 SA n}{\delta} \right) \max_{s,a} \mathcal{F}(s,a; T_{\pi,n}),
\end{align*}
where
\begin{align*}
\mathcal{F}(s,a; T) :=  \frac{ V(s,a) }{\sqrt{T(s,a) + 1}} + \frac{S}{T(s,a) + 1}.
\end{align*}
\end{proposition}

This upper bound provides a much more explicit dependency between the behavior of $\pi$, the structure of the MDP, and the accuracy of the model estimation. In particular, it illustrates how rarely visited pairs with large variance lead to larger estimation errors.

\textbf{Convex reparametrization.} Prop.~\ref{p:upper.bound} suggests that $\calE_{\pi,n}$ (resp.~$\calW_{\pi,n}$) could be optimized by directly minimizing the average (resp.~the maximum) of its upper-bound $\mathcal{F}(s,a;T)$. While the function $\mathcal{F}$ is convex in $T$, the constraint that $T_{\pi,n}$ must be a valid state-action counter renders the optimization problem of minimizing $\mathcal{F}(T_{\pi,n})$ non-convex. To circumvent this issue, we introduce the state-action stationary distribution $\lambda_{\pi} \in \Delta(\SA)$ of a stationary policy $\pi$. Also known as an occupancy measure, $\lambda_{\pi}(s,a)$ is the expected frequency with which action $a$ is executed in state $s$ while following policy $\pi$. Let $\Lambda^{(p)}$ be the set of state-action stationary distributions, i.e.,
\begin{align*}
\Lambda^{(p)} := \Big\{ &\lambda \in \Delta(\SA): \forall s \in \mathcal{S}, \\
& \sum_{b \in \A } \lambda(s,b) = \sum_{(s',a) \in \SA} p(s \vert s',a) \lambda(s',a) \Big\}.
\end{align*}
We introduce a convenient function
\begin{align*}
\calG_n(s,a;\lambda) := \frac{ V(s,a) }{\sqrt{\lambda(s,a) +  \frac{1}{n}}} + \frac{1}{\sqrt{n}} \frac{S}{\lambda(s,a) + \frac{1}{n}},
\end{align*}
and its associated functions
\begin{align*}
\mathcal{L}_n^\calE(\lambda) := \frac{1}{SA} \sum_{s,a} \calG_n(s,a; \lambda),\\
\mathcal{L}_n^\calW(\lambda) := \max_{s,a} \calG_n(s,a; \lambda).
\end{align*}
Let the empirical state-action frequency at any time $t$ be $\widetilde{\lambda}_{\pi,t}(s,a) = T_{\pi,t}(s,a)/t$, then it is easy to see that
\begin{align*}
\mathcal{E}_{\pi,n} & \leq \frac{4}{\sqrt{n}} \log\left(\frac{6 SA n}{\delta} \right) \mathcal{L}_n^\calE(\wt\lambda_{\pi,n}),\\
\mathcal{W}_{\pi,n} & \leq \frac{4}{\sqrt{n}} \log\left(\frac{6 SA n}{\delta} \right) \mathcal{L}_n^\calW(\wt\lambda_{\pi,n}).
\end{align*}
This suggests that we could directly minimize $\mathcal{E}_{\pi,n}$ (resp.~$\calW_{\pi,n}$) by optimizing the function $\calL_n^\calE$ (resp.~$\calL_n^\calW$) over $\lambda$ in the set of stationary distributions $\Lambda^{(p)}$ as follows
\begin{equation}
\begin{aligned}
\min_{\lambda \in \Lambda^{(p)}} \mathcal{L}_n^\calE(\lambda), \quad\quad \min_{\lambda \in \Lambda^{(p)}} \mathcal{L}_n^\calW(\lambda).
\end{aligned}
\tag{$\dagger$}
\label{optimization_problem}
\end{equation}
These functions are convex in both the objective function and the constraints and we denote by $\lambda^\calE_{\dagger,n}$ and $\lambda^\calW_{\dagger,n}$ their solutions. Roughly speaking, the optimal distributions should visit more often state-action pairs associated to a larger transitional noise. 

Beside $\calL_n$, we also use closely related functions obtained as asymptotic relaxation for $n\rightarrow \infty$ and define the optimization problems
\begin{equation}
\begin{aligned}
\min_{\lambda \in \Lambda^{(p)}} \mathcal{L}_{\infty}^\calE(\lambda) &:= \frac{1}{SA} \sum_{s,a} \frac{V(s,a)}{\sqrt{\lambda(s,a)}},\\
\min_{\lambda \in \Lambda^{(p)}} \mathcal{L}_{\infty}^\calW(\lambda) &:= \max_{s,a} \frac{V(s,a)}{\sqrt{\lambda(s,a)}}.
\end{aligned}
\tag{$\star$}
\label{asymptotic_optimization_problem}
\end{equation}
We refer to $\lambda^\calE_\star$ and $\lambda^\calW_\star$ the corresponding optimal state-action stationary distributions.


\vspace{-0.1in}
\section{LEARNING TO OPTIMIZE \MODESTtitle}
\label{section_exact_objective}
\vspace{-0.1in}

In this section, we build on~\citep{tarbouriech2019active,berthet2017fast}, introduce a learning algorithm for the \MODEST problem and prove its sample complexity for both $\calE$ and $\calW$ error functions. As the structure of the algorithm is the same for both cases, we describe it in its general form.


\vspace{-0.05in}
\subsection{Setting the Stage}\label{ssec:alg.preparation}
\vspace{-0.05in}

Leveraging~\eqref{optimization_problem} to design an algorithm minimizing the estimation error requires addressing two issues: \textbf{1)} the transitional noise $V(s,a)$ in the objective function and the transition dynamics $p$ appearing in the constraint are unknown, \textbf{2)} despite $\calL_n^\calE$ and $\calL_n^\calW$ being convex, they are both poorly conditioned and difficult to optimize.

\textbf{Parameters estimation.}
After $t$ steps, the algorithm constructs estimate $\wh p_t$ as in Eq.~\eqref{eq:emp.estimate} using the samples collected so far. Furthermore, it defines the empirical variance of the transition from $(s,a)$ to $s'$ as $\widehat{\sigma}^2_{t}(s'\vert s,a) := \widehat{p}_{t}(s' \vert s,a)(1-\widehat{p}_{t}(s' \vert s,a))$. Then we have the following high-probability guarantees on $\wh p_t$ and an upper-confidence bound on the transitional noise.

\begin{lemma}\label{lem:hp.events}
Define $\calB_t(s,a,s') := \{ \widetilde{p} \in \Delta(S) : \abs{\widetilde{p}(s' \vert s,a) - \widehat{p}_t(s' \vert s,a)} \leq B_t(s,a,s')\}$ and
\begin{align*}
B_{t}(s,a,s') := 2\sqrt{\frac{\wh{\sigma}_t^2(s' \vert s,a)}{T_t^+(s,a)} \ell_t} + \frac{6 \ell_t}{T_t^+(s,a)},
\end{align*}
where $\ell_t := \log\Big(\frac{6 S A T_t^+(s,a)}{\delta}\Big)$. Furthermore, let
\begin{align*}
\wh{V}^+_t(s,a) := \frac{1}{\sqrt{S}} \sum_{s' \in \mathcal{S}} \Big( \sqrt{\widehat{\sigma}^2_{t}(s'\vert s,a)} + \sqrt{2\frac{\ell'_{t}}{T^+_{t}(s,a)}} \Big),
\end{align*}
where $\ell'_t = \log\left(\frac{4 S^2 A (T^+_{t}(s,a))^2}{\delta}\right)$. Then with probability at least $1-\delta$, for any $t>0$ and $(s,a) \in \mathcal{S} \times \mathcal{A}$, we have
\begin{align*}
V(s,a) \leq \wh{V}^+_t(s,a), \quad p(\cdot|s,a)\in\calB_t(s,a).
\end{align*}
\end{lemma}

This result allows us to build upper bounds on the error functions and correctly manage the constraint set $\Lambda^{(p)}$.


\textbf{Smooth optimization.}
We first notice that both $\calL_n^\calE$ and $\calL_n^\calW$ are built on $\calG_n$ which may grow unbounded when $n$ grows to infinity and $\lambda$ tends to zero in state-action pairs with non-zero transitional noise. Furthermore, the function is smooth with coefficient scaling as $n^{5/2}$, which diverges as $n \rightarrow + \infty$. This means that optimizing $\calL_n$ functions may become increasingly more difficult for large $n$. In order to avoid this issue we restrict the space of $\lambda$. For any $\eta>0$ we introduce the restricted simplex
\begin{align*}
\Lambda^{(p)}_\eta := \{ \lambda \in \Lambda^{(p)} ~\vert~ \forall (s,a) \in \mathcal{S} \times \mathcal{A}, \lambda(s,a) \geq \eta \}.
\end{align*}
On this set we have that $\calG_n$ is now bounded and smooth independently from $n$. Since $\calL_n^\calE$ is just an average of functions $\calG_n$, it directly inherits these properties.
\begin{proposition}\label{prop:restricted.function}
On the set $\Lambda^{(p)}_\eta$ the function $\calL_n^\calE$ has a smoothness constant scaling as $1/\eta^3$.
\end{proposition}
Unfortunately, the same guarantees do not hold for $\calL_n^\calW$, as it is a non-smooth function of $\calG_n$. We thus modify $\calL_n^\calW$ using a LogSumExp (\LSE) transformation. We briefly recall its properties~\citep[see e.g.,][]{beck2017first}.
\begin{proposition}\label{prop.logsumexp}
	For any $x \in (\mathbb{R}_+)^m$, let $\LSE(x) := \log\big( \sum_{i=1}^m \exp(x_i) \big)$. Then the following properties hold:
	\begin{itemize}[leftmargin=.2in,topsep=-2.5pt,itemsep=1pt,partopsep=0pt, parsep=0pt]
		\item[1)] \LSE~is convex and $1$-smooth w.r.t.~the $\ell_{\infty}$-norm.
		\item[2)] $\max_{i \in [m]} x_i \leq \LSE(x) \leq \max_{i \in [m]} x_i + \log(m)$.
	\end{itemize}
	\label{lemma_logsumexp}
\end{proposition}
While the first property ensures convexity and smoothness, the second statement shows that the bias introduced by the transformation is logarithmic in the number of elements. We define the \LSE~version of $\calL_n^\calW$ as
\begin{align}
\wb\calL_n^\calW(\lambda) := \log\Big( \sum_{s,a} \exp(\calG_n(s,a;\lambda)) \Big),
\end{align}
%
which satisfies the following property by the chain rule and Prop.~\ref{prop.logsumexp}.
\begin{proposition}\label{prop:logsumexp.w}
On the set $\Lambda^{(p)}_\eta$ the function $\wb\calL_n^\calW$ has a smoothness constant scaling as $1/\eta^5$.
\end{proposition}
Comparing Prop.~\ref{prop:logsumexp.w} to Prop.~\ref{prop:restricted.function}, we notice that beside being biased, $\wb\calL_n^\calW$ is also less smooth than $\calL_n^\calE$ and it is thus possibly harder to optimize.


\vspace{-0.05in}
\subsection{The Algorithm}
\label{ssec:algorithm}
\vspace{-0.05in}


\begin{algorithm}[t!]
	\begin{small}
	\textbf{Input:} Constraint parameter $\eta > 0$, confidence $\delta\in (0,1)$, budget $n$, objective function $\calL_n$ (either $\calL_n^\calE$ or $\wb \calL_n^\calW$). \\
	\textbf{Initialization:} Set $t := 1$, state-action counter $T_0(s,a) := 0$ and empirical frequency $\wt{\lambda}_0(s,a) := \frac{1}{SA}$. \\
	\For{episode $k=1,2,\ldots$}
	{
		Set $t_k := t$. \\
		Compute the upper-confidence bound $\wh{V}^{+}_k$  to the estimates of the transitional noise. \\
		Compute the objective function $\calL_n$ using $\wh V^+_k$.\\
		Compute the occupancy measure $\wh{\phi}_{k+1}^+$ using~\eqref{eq:fw.estimate}. \\
		Derive the policy $\wt{\pi}_{k+1}^+$ using~\eqref{eq:new.policy}. \\
		Execute $\wt{\pi}_{k+1}^+$ for $\tau_k$ steps. \\
		Update the time index $t$ and empirical frequency $\wt{\lambda}_{k+1}$. \\
	}
	\end{small}
	\caption{\FWMODEST}
	\label{algorithm_fw_modest}
\end{algorithm}


We build on the estimation algorithm of~\citep{tarbouriech2019active} and propose \FWMODEST (Alg.~\ref{algorithm_fw_modest}), which proceeds through episodes $k=1,2,\ldots$ and applies a Frank-Wolfe approach to minimize the model estimation error. The algorithm receives as input the state-action space, the parameter $\eta$ used in defining the restricted space $\Lambda$, the budget $n$, and the objective function $\calL_n$ to optimize (i.e., either $\calL_n^\calE$ or the smoothed $\wb\calL_n^\calW$).\footnote{In the following we use $\calL_n$ for the generic function the algorithm is required to optimize.} We first review how $\calL_n$ could be optimized by the Frank-Wolfe (\FW) algorithm in the exact case, when the transition dynamics $p$ and transitional noise $V(s,a)$ are known. Let $\lambda_0$ be the initial solution in $\Lambda_\eta^{(p)}$, then \FW proceeds through iterations by computing
\begin{align}\label{eq:fw.exact}
\phi_{k+1} = \arg\min_{\lambda \in \Lambda^{(p)}} \langle \nabla \calL_n(\lambda_k), \lambda \rangle,
\end{align}
which is then used to update the candidate solution as $\lambda_{k+1} = \lambda_k + \beta_k (\phi_{k+1} - \lambda_k)$. Whenever the learning rate $\beta_k$ is properly tuned, the overall process is guaranteed to converge to optimal value $\calL_n(\lambda_{\dagger,n})$ at the rate $O(\kappa_\eta/k)$~\citep{jaggi2013revisiting}, where $\kappa_\eta$ is the smoothness of the function (see Prop.~\ref{prop:restricted.function} and~\ref{prop:logsumexp.w}).
\FWMODEST integrates the Frank-Wolfe scheme into a learning loop, in which the solution constructed by the agent at the beginning of episode $k$ is the empirical frequency of visits $\wt \lambda_{k}$, and the optimization in~\eqref{eq:fw.exact} is replaced by an estimated version where the upper bound $V^+_{k}(s,a)$ from Lem.~\ref{lem:hp.events} is used in place of the transitional noise $V(s,a)$. Moreover, the confidence interval $\calB_{k}$ over the transition dynamics is used to characterize the constraint $\Lambda^{(p)}$, thus leading to the optimization\footnote{The $\argmin$ extracts only the solution $\lambda$.}
\begin{align}\label{eq:fw.estimate}
\wh\phi_{k+1}^+ = \argmin_{\wt p\in\calB_{k}, \lambda \in \Lambda^{(\wt p)}} \langle \nabla \wh \calL_{n,k}^+(\wt\lambda_k), \lambda \rangle,
\end{align}
where $\wh \calL_{n,k}^+$ is $\calL_n$ with $V$ replaced by $\wh V^+_{k}$ in the formulation of $\calG_n$. The optimization in \eqref{eq:fw.estimate} can be seen as the problem of solving an MDP with an optimistic choice of the dynamics within the confidence set $\calB_k$ and where the reward is set to $-\nabla \wh \calL_{n,k}^+(\wt\lambda_k)$, i.e., state-action pairs with larger (negative) gradient have larger reward. This directly leads to solutions that try to increase the accuracy in state-action pairs where the current estimation of the model is likely to have a larger error. Notice that the output of the optimization is the distribution $\wh\phi_{k+1}^+$, which cannot be directly used to update the solution $\lambda_k$ as in \FW, since the empirical frequency can only be updated by executing a policy collecting new samples. Thus, $\wh\phi_{k+1}^+$ is used to define a policy $\pi_{k+1}^+$ as
\begin{align}\label{eq:new.policy}
\wt{\pi}_{k+1}^+(a \vert s) := \frac{\wh{\phi}_{k+1}^+(s,a)  }{ \sum_{b \in \mathcal{A}} \wh{\phi}_{k+1}^+(s,b)},
\end{align}
which is then executed for $\tau_k$ steps until the end of the current episode. The samples collected throughout the episode are then used to compute the new frequency $\wt\lambda_{k+1}$.

\textbf{The extended LP problem.}
While the overall structure of \FWMODEST is similar to the estimation algorithm in~\citep{tarbouriech2019active}, the crucial difference is that the true dynamics $p$ is unknown and~\eqref{eq:fw.estimate} cannot be directly solved as an LP. In fact, while for any fixed $\wt p$ the set $\Lambda^{(\wt p)}$ only poses linear constraints and optimizing over $\lambda$ coincides with the standard dual LP to solve MDPs~\citep{puterman2014markov}, in our case we have to list all possible values of $\wt p$ in $\calB_k$. We thus propose to rewrite~\eqref{eq:fw.estimate} as an \textit{extended} LP problem by considering the state - action - next-state occupancy measure $q(s,a,s')$, defined as $q(s,a,s') := p(s'\vert s,a) \lambda(s,a)$.\footnote{
	This construction resembles the extended LP used for loop-free stochastic shortest path problems in~\citep{rosenberg2019online}.
}  We recall from Lem.~\ref{lem:hp.events} that at any episode $k$, we have an estimate $\wh{p}_k$ such that, w.h.p., for any $(s,a,s') \in \SAS,$
%
$\abs{\wh{p}_k(s' \vert s,a) - p(s' \vert s,a)} \leq B_k(s,a,s')$.
%
For notational convenience, we define $r_k(s,a) := -\nabla \wh \calL_{n,k}^+(\wt\lambda_k)(s,a)$. We then introduce the extended LP associated to~\eqref{eq:fw.estimate} formulated over the variables $q$ as

\vspace{-0.15in}
\begin{small}
	\begin{align*}
	&\max_{q\in \Delta(\mathcal S\times\A\times\calS)} \sum_{s,a,s'} r_k(s,a) q(s,a,s')  \\
	\mbox{s.t.} \ & \forall j\in\calS, \sum_{a,s} q(j,a,s) - \sum_{s,a} q(s,a,j) = 0, \\
	& q(s,a,j) - (\wh{p}_k(j|s,a) + B_k(s,a,j)) \sum_{s'} q(s,a,s') \leq 0,\\
	& -q(s,a,j) + (\wh{p}_k(j|s,a) - B_k(s,a,j)) \sum_{s'} q(s,a,s') \leq 0,\\
	& \sum_{s'} q(s,a,s') \geq \eta.
	\end{align*}
\end{small}
\vspace{-0.15in}

Crucially, the reparametrization in $q$ enables to efficiently solve the problem as a ``standard'' LP. Let $q_{k+1}$ be the solution of the problem above, then the state-action distribution can be easily recovered as
%
$\wh\phi_{k+1}^+(s,a) := \sum_{s'} q_{k+1}(s,a,s')$,
%
and the corresponding policy is computed as in~\eqref{eq:new.policy}.



\vspace{-0.05in}
\subsection{Sample Complexity}
\label{ssec:sample.complexity}
\vspace{-0.05in}

In order to derive the sample complexity of \FWMODEST, we build on the analysis in~\citep{tarbouriech2019active} and show that the algorithm controls the regret w.r.t.\ the optimal solution of the problem $\min\calL_n(\lambda)$, i.e.,
\begin{align}\label{eq:regret.modest}
\mathcal{R}_n := \mathcal{L}_n(\widetilde{\lambda}_{n}) - \mathcal{L}_n(\lambda_{\dagger,n}),
\end{align}
where, similar to the previous section, $\calL_n$ can be either $\calL_n^\calE$ or $\wb \calL_n^\calW$, $\lambda_{\dagger,n}$ is the corresponding optimal solution and $\wt\lambda_n$ is the empirical frequency of visits after $n$ steps.
We require $\eta$ and the MDP to satisfy the two following assumptions.\footnote{Refer to~\citep{tarbouriech2019active} for more details about the assumptions.}

\begin{assumption}\label{asm:eta}
	The parameter $\eta$ is such that
	\begin{align*}
	\lambda_{\dagger,n} \in \arg\min_{\lambda \in \Lambda_{\eta}^{(p)}} \mathcal{L}_n(\lambda), \quad \lambda_{\star} \in \arg\min_{\lambda \in \Lambda_{\eta}^{(p)}} \mathcal{L}_{\infty}(\lambda),
	\end{align*}
	i.e., the optimal solutions in the restricted set $\Lambda$ coincide with the unrestricted solutions for both $\calL_n$ and $\calL_\infty$.
\end{assumption}

\begin{assumption}\label{asm:mdp}
	The MDP $M$ is ergodic. We denote by $\gamma_{\min} := \min_{\pi} \gamma_{\textrm{ps}}^{\pi}$ the smallest pseudo-spectral gap~\citep{paulin2015concentration} over all stationary policies $\pi$ and we assume $\gamma_{\min} > 0$.
	\label{asm_ergodicity}
\end{assumption}

We are now ready to derive the model estimation error and sample complexity of \FWMODEST. By extending the analysis from~\citep{tarbouriech2019active} to unknown transition model and integrating the regret analysis into our estimation error problem, we obtain the following (see App.~\ref{app_A} for the proof and the explicit dependencies).

\begin{theorem}\label{thm:mod.est}
If \FWMODEST is run with a budget $n$ and the length of the episodes is set to $\tau_k = 3k^2 - 3k + 1$, then under Asm.~\ref{asm:eta} and~\ref{asm:mdp}, w.p. $1-\delta$ and depending on the optimization function $\calL_n$ given as input, we have that
\begin{align*}
\mathcal{E}_{\pi,n} &= \wt{O} \left( \frac{\mathcal{L}^\calE_{\infty}(\lambda_{\star}^\calE)}{\sqrt{n}} + \frac{\Theta^\calE}{n^{5/6}} \right), \\
\mathcal{W}_{\pi,n} &= \wt{O} \left( \frac{{\mathcal{L}}^\calW_{\infty}(\lambda_{\star}^\calW) + \log(SA)}{\sqrt{n}} + \frac{ \Theta^\calW }{n^{5/6}} \right),
\end{align*}
%
where $\Theta^\calE$ (resp.~$\Theta^\calW$) scales polynomially with MDP constants $S$, $A$, and $\gamma_{\min}^{-1}$, and with algorithmic dependent constants $\log(1/\delta)$ and $\eta^{-1}$ when optimizing $\calE$ (resp.~$\calW$). From this result we immediately obtain the sample complexity of \FWMODEST
%
\begin{align*}
\mathcal{C}^{\mathcal{E}}(\epsilon, \delta) &= \Omega\left( \frac{\mathcal{L}^\calE_{\infty}(\lambda_{\star}^\calE)^2}{\epsilon^2} + \frac{\Theta^\calE}{\epsilon^{6/5}} \right),\\
\mathcal{C}^{\mathcal{W}}(\epsilon, \delta) &= \Omega\left( \frac{\big(\mathcal{L}^\calW_{\infty}(\lambda_{\star}^\calW) + \log(SA)\big)^2}{\epsilon^2} + \frac{\Theta^\calW}{\epsilon^{6/5}} \right).
\end{align*}
\end{theorem}

\textbf{Model estimation performance.} The previous result shows that \FWMODEST successfully targets the model estimation problem with provable guarantees. In particular, the model estimation errors above display a leading error term scaling as $\wt O(n^{-1/2})$ and a lower-order term scaling as $\wt O(n^{-5/6})$. Under Asm.~\ref{asm:mdp}, any stochastic policy $\pi$ with non-zero probability to execute any action is guaranteed to reach all states with a non-zero probability. As a result, the associated model estimation error would decrease as $\wt O(n^{-1/2})$ as well, since $T_{\pi,n}(s,a)$ grows linearly with $n$ in all state-action pairs. Nonetheless, the linear growth could be very small as it depends on the smallest probability of reaching any state by following~$\pi$.  As a result, despite the fact that~$\pi$ may achieve the same rate, \FWMODEST actually performs much better. In particular, a trivial strategy like the uniform policy $\pi_{\textrm{unif}}$ yields a main-order term decreasing in $\calL_\infty(\lambda_{\pi_{\textrm{unif}}})/n^{1/2}$, whereas the term  $\calL_\infty(\lambda_\star)/n^{1/2}$ achieved in Thm.~\ref{thm:mod.est} is an upper bound to the smallest error that can be achieved for the specific MDP at hand by definition of $\lambda_\star$ (this gap between the uniform policy and \FWMODEST is experimentally displayed in Sect.~\ref{sec_experiments}). As such, Thm.~\ref{thm:mod.est} shows that \FWMODEST is able to adapt to the current problem and decrease the model estimation error at the best possible rate, up to an additive error (the lower-order term), which is decreasing to zero at the faster rate $\wt O(n^{-5/6})$.

\textbf{Limitations.} Despite its capability of tracking the performance of the best state-action distribution, \FWMODEST and its analysis suffer from several limitations. Asm.~\ref{asm:mdp} poses significant constraints both on the choice of $\eta$ and the ergodic nature of the MDP. Moreover, the lower-order terms in Thm.~\ref{thm:mod.est} depend inversely on the parameter~$\eta$ (via the optimization properties of $\calE$ and $\calW$, see App.~\ref{app_A} \ding{175}),
which implicitly worsens the dependency on the state-action size, to the extent that the second term may effectively dominate the overall error even for moderately big values of~$n$.
This drawback is even stronger in the case of $\calW$. In fact, as shown in Prop.~\ref{prop.logsumexp}, the function $\wb\calL_n^\calW$ is biased by $\log(SA)$, which is then reflected into the final accuracy of \FWMODEST. While it is possible to change the definition of $\wb{\calL}^\calW_\infty$ to reduce the bias, this would lead to a less smooth function, which would make the constants in the lower-order term even larger. We also notice that as~$\lambda$ appears in the denominator of all functions $\calL$, optimizing $\calL$ often becomes numerically unstable for large state-action space, where $\min_{s,a} \lambda(s,a) \leq 1/SA$. Finally, the episode length choice in Thm.~\ref{thm:mod.est} is often conservative in practice, where shorter episodes usually perform better.

While it remains an open question whether it is possible to achieve better results for the \MODEST problem, in the next section we introduce a heuristic objective function that allows more efficient learning with looser assumptions.


\vspace{-0.1in}
\section{WEIGHTED \MaxEnttitle}
\label{section_approximate_objective}
\vspace{-0.1in}

We introduce a heuristic objective function for model estimation and we propose a variant of the algorithm in~\citep{cheung2019arvixexploration} for which we prove regret guarantees.

\vspace{-0.05in}
\subsection{Maximum Weighted Entropy}
\vspace{-0.05in}

A good exploration strategy to traverse the environment is to maximize the entropy of the empirical state frequency: this consists in the \MaxEnt algorithm introduced in~\citep{hazan2019provably}. Yet this strategy aims at generating a uniform coverage of the state space. While it may be beneficial to perform the \MaxEnt strategy in some settings, this is undesirable for the \MODEST objective. First, the state-action space should be considered instead of only the state space. Second and crucially, whenever there is a discrepancy in the transitional noise, each state-action pair should not be visited uniformly as often. Indeed, the convex relaxation in Sect.~\ref{sec_problem_formulation} emphasizes that in order to minimize the model estimation error the agent should aim at visiting each state-action pair $(s,a)$ proportionally to its transitional noise $V(s,a)$. A way to do this is to consider a \textit{weighted} entropy objective function, with the weight of each state-action pair depending on its transitional noise.


\begin{definition}
    For any non-negative weight function $w : \mathcal{S} \times \mathcal{A} \rightarrow \mathbb{R}^+$, the weighted entropy $H_w$ is defined on $\Delta(\mathcal{S} \times \mathcal{A})$ as follows:
    \begin{align*}
        H_w(\lambda) := - \sum_{(s,a) \in \mathcal{S} \times \mathcal{A}} w(s,a) \lambda(s,a) \log\left(\lambda(s,a)\right).
    \end{align*}
\end{definition}
To gain insight on $H_w$, we can argue that $w(s,a)$ represents the value or utility of each outcome $(s,a)$~\citep[see e.g.,][]{guiacsu1971weighted}. Hence, the learner is encouraged to allocate importance on information regarding $(s,a)$ proportionally to the weight $w(s,a)$. This can be translated in biasing exploration towards region of the state-action space with high weights $w(s,a)$. Similar to~\citep{hazan2019provably}, we also introduce a smoothed version of $H_w$ and show its properties.

\begin{lemma}
	For any arbitrary set of non-negative weights $w$ --- for which we define $W := \max_{s,a}w(s,a)$ --- and for any smoothing parameter $\mu > 0$, we introduce the following smoothed proxy
	\begin{align*}
	H_{w,\mu}(\lambda) := \sum_{s,a} w(s,a) \lambda(s,a) \log\left(\frac{1}{\lambda(s,a) + \mu}\right).
	\end{align*}
	The function $H_{w, \mu}$ satisfies the following properties:
	\begin{enumerate}[leftmargin=.2in,topsep=-2.5pt,itemsep=1pt,partopsep=0pt, parsep=0pt]
		\item \small{$[ \nabla H_{w,\mu}(\lambda)]_{s,a} = - w(s,a) \left( \log(\lambda(s,a) + \mu) + \frac{\lambda(s,a)}{\lambda(s,a) + \mu} \right)$.}
		\item \normalsize{For any $\mu \leq \frac{1}{e} - \frac{1}{SA}$, we have $[ \nabla H_{w,\mu}(\lambda)]_{s,a} \geq 0$.}
		\item \normalsize{$H_{w, \mu}$ is concave in $\lambda$, as well as $W \log\left( \frac{1}{\mu}\right)$-Lipschitz continuous and $\frac{2 W}{ \mu}$-smooth.}
		\item \normalsize{We have $\abs{H_w(\lambda) - H_{w,\mu}(\lambda)} \leq \mu S A W$.}
	\end{enumerate}
	\label{lemma_properties_weighted_entropy}
\end{lemma}


In light of the intuition gained by inspecting the definition of $\calL_n$, in the following we set the weights $w(s,a) := V(s,a)/\sqrt{SL}$, with $L := \log(SA/\delta)$, so as to encourage visiting state-action pairs with large transitional noise. We notice that, unlike the smooth versions of $\calE$ and $\calW$, $H_{w,\mu}$ has a much smaller smoothness constant and its bias can be easily controlled by the choice of $\mu$. Moreover, the state-action distribution $\lambda$ does not appear in the denominator as in $\calL_n$. All these factors suggest that learning how to optimize $H_{w,\mu}$ may be much simpler than directly targeting the model estimation error.





\subsection{Learning to Optimize the Weighted Entropy}
\label{subsection_generalization_cheung}


\begin{algorithm}[t!]
	\begin{small}
	\textbf{Input:} Smoothing parameter $\mu$, confidence $\delta \in (0,1)$. \\
	\textbf{Initialization:} Set $t := 1$, state-action counter $T_0(s,a) := 0$ and empirical frequency $\wt{\lambda}_0(s,a) := \frac{1}{SA}$. \\
	Compute gradient threshold $Q := 2 \log\left(\frac{1}{\mu}\right)$. \\
	\For{episode $k=1,2,\ldots$}
	{
		Set $t_k := t$. \\
		Compute the upper-confidence bound $\wh{V}^{+}_k$  to the estimates of the transitional noise. \\
		Compute the weighted entropy $\wh{\mathcal{H}}_k^+$ from~\eqref{weighted_entropy_computation}. \\
		Compute a near-optimal policy $\wh\pi_{k+1}^+ := \EVI(\nabla \wh{\mathcal{H}}_k^+(\wt{\lambda}_{k}), \mathcal{B}_k, \frac{1}{\sqrt{t_k}})$. \\
		Initialize $\nu_k(s,a) := 0$, $\Phi := 0$ and $\theta_k^{\textrm{ref}} := \wt{H}^{+}_k(\wt{\lambda}_{t_k})$. \\
		\While{$\Phi \leq Q$ and $\nu_k(s,a) < T_k(s,a)$}
		{
			Execute action $a_t = \wh\pi^+_{k+1}(s_t)$, observe the next state $s_{t+1}$. \\
			Compute gradient $\theta_{t+1} := \nabla \wh{\mathcal{H}}_k^+(\wt{\lambda}_t)$. \\
			Update $\Phi \mathrel{+}= \norm{ \theta_{t+1} - \theta_k^{\textrm{ref}}}_2$. \\
			Update $\nu_k(s_t,a_t) \mathrel{+}= 1$, $\wt{\lambda}_{t+1} := \frac{t}{t+1} \wt{\lambda}_t + \frac{1}{t+1} \mathds{1}_{s_t,a_t}$, and $t \mathrel{+}= 1$. \\
		}
		Set $T_{k+1}(s,a) := T_k(s,a) + \nu_k(s,a)$, update $\wt\lambda_{k+1}$. \\
	}
	\end{small}
	\caption{\WeightedMaxEnt w/ \TOCUCRLtwo}
	\label{algorithm_weighted_maxent}
\end{algorithm}

We seek to design a learning algorithm maximizing $H_w$. Since $\mathcal{H}_{w,\mu}$ is concave, Lipschitz continuous and smooth in $\lambda$, the algorithm of \citep{cheung2019arvixexploration} could have been readily applied to maximize it \textit{if} the function was known. The key difference here is that the weights are \textit{unknown}. We thus generalize the algorithm of~\citep{cheung2019arvixexploration} to handle this case. The resulting algorithm is outlined in Alg.~\ref{algorithm_weighted_maxent}. In it, {\small\textsc{EVI}}($r,\mathcal{B},\varepsilon$) denotes the standard extended value iteration scheme for reward function $r$, confidence region $\mathcal{B}$ around the transition probabilities, and accuracy $\varepsilon$ (see \citep{cheung2019arvixexploration} for more details).

Following the terminology of \citep{cheung2019arvixexploration}, in our setting the vectorial outcome at any state-action pair $(s,a)$ is the standard basis vector for $(s,a)$ in $\mathbb{R}^{S \times A}$ (i.e., $\mathds{1}_{s,a}$). The objective is to minimize the regret
\begin{equation}\label{eq:regret.entropy}
\mathcal{R}_n^{Ent} := \min_{\lambda\in\Lambda^{(p)}}\mathcal{H}_w(\lambda) - \mathcal{H}_w(\wt{\lambda}_t),
\end{equation}
where $\Lambda^{(p)}$ is the space of state-action strationary distributions and $\wt\lambda_t$ is the empirical frequency of visits returned by the algorithm after $t$ steps.
Notice that unlike in Sect.~\ref{sec_problem_formulation} and~\ref{section_exact_objective}, the optimum is computed over $\lambda$ in the unrestricted set of stationary distribution (i.e., no $\eta$ lower bound), thus making this regret more general and challenging than~\eqref{eq:regret.modest}.



We extend the scope of \TOCUCRLtwo~\citep{cheung2019arvixexploration} to handle function $\mathcal{H}_{w,\mu}$. For any time step $t$, let us define

\vspace{-0.1in}
\begin{small}
\begin{align}\label{weighted_entropy_computation}
    \wh{\mathcal{H}}_t^+(\lambda) := \sum_{s,a} \frac{\wh{V}^+_{t}(s,a)}{\sqrt{SL}} \lambda(s,a) \log\left(\frac{1}{\lambda(s,a) + \mu} \right).
\end{align}
\end{small}
\vspace{-0.1in}

Then we can derive a bound similar to Lem.~\ref{lem:hp.events}.

\begin{lemma}\label{lemma_deviation_gradient}
	 With probability at least $1-\delta$, at any time step $t$, we have component-wise
	\begin{align*}
	\nabla \mathcal{H}(\wt{\lambda}_t) \leq \nabla \wh{\mathcal{H}}_t^+(\wt{\lambda}_t) \leq \nabla \mathcal{H}(\wt{\lambda}_t) + u(t, \delta),
	\end{align*}
	where the deviation $u(t, \delta)$ satisfies
	\begin{align*}
	[u(t, \delta)]_{s,a} = O\bigg( \sqrt{ \frac{\log(\frac{S A t}{\delta})}{T_{t}(s,a)}} \bigg).
	\end{align*}
\end{lemma}



We now show how a simple extension of the algorithm of \citep{cheung2019arvixexploration} can cope with the setting of unknown weights. At the beginning of each episode $k$, instead of feeding as reward the true, unknown current gradient, we simply feed an upper confidence bound of it, i.e., $\nabla \wh{\mathcal{H}}_k^+(\wt{\lambda}_k)$, where the optimistic objective function $\wh{\mathcal{H}}_k^+$ --- which depends on the smoothing parameter $\mu$ --- is defined as in~\eqref{weighted_entropy_computation}.  Initially, for small values of $t$, the weights are very similar, thus the algorithm targets a uniform exploration over the state-action space, that is, the original \MaxEnt objective. As more samples are collected, the estimation of the weights becomes more and more accurate, thus the exploration is gradually skewed towards regions of the state-action space with high transitional noise.

It is possible to extend the analysis of \TOCUCRLtwo to this case and obtain a similar regret bound. As in \citep{cheung2019arvixexploration}, the only assumption we need is that the MDP $M$ is communicating, a weaker requirement than Asm.~\ref{asm:mdp}.

\begin{assumption}\label{asm_communicating}
	The MDP $M$ is communicating, with diameter $D := \max_{s \neq s'} \min_{\pi} \mathbb{E}\left[\tau_{\pi}(s \rightarrow s')\right] < +\infty$.
\end{assumption}

Then we can prove the following result.

\begin{theorem}\label{thm.weighted.entropy}
If \WeightedMaxEnt is run with a budget $n$ and the smoothing parameter set to $\mu = \frac{1}{n^{1/3} S^{2/3}}$, then under Asm.~\ref{asm_communicating}, with overwhelming probability, we have
\begin{align}\label{eq:regret.wentropy}
\mathcal{R}_n^{Ent} = \wt{O}\left( \frac{D S^{1/3}}{n^{1/3}} + \frac{D \sqrt{\Gamma S A}}{\sqrt{n}} \right).
\end{align}
\end{theorem}

While they target different objectives, it is interesting to do a qualitative comparison between Thm.~\ref{thm.weighted.entropy} and Thm.~\ref{thm:mod.est}. We first extract from the error bound in Thm.~\ref{thm:mod.est} that \FWMODEST has a regret scaling as $\Theta/n^{1/3}$ w.r.t.~$\calL_\infty^\calE(\lambda_\star^\calE)$. While the rate is the same as in~\eqref{eq:regret.wentropy}, the difference is that the constants are much better for the weighted \MaxEnt case. In fact, it scales linearly with the diameter $D$ and sublinearly with the size of the state-action space, instead of high-order polynomial dependencies on constants such as the constraint parameter $\eta$ and the MDP parameter $\gamma_{\min}$, which are likely to be very small in any practical application. As a result, we expect the weighted \MaxEnt version of \TOCUCRLtwo to approach the performance of the optimal weighted entropy solution $\arg\min_{\lambda\in\Lambda^{(p)}}\mathcal{H}_w(\lambda)$ faster than \FWMODEST approaches the performance of $\lambda_\star$.

While we do not have an explicit link between the model estimation error and the weighted entropy, we notice an important connection between $\calL_\infty^\calE$ and $H_w$ in the unconstrained case (i.e., when $\Lambda$ is reduced to the simplex over the state-action space with no constraint from the MDP). In this case, the solution to $\min_{\lambda} H_w(\lambda)$ is $\lambda_\star^{\WeightedMaxEnt}(s,a) \propto \exp(\alpha V(s,a))$, with $\alpha$ a suitable constant, which shows that the optimal distribution has a direct connection with the transitional noise $V(s,a)$. This exactly matches the intuition that a good distribution to minimize $\calL_\infty^\calE$ should spend more time on states with larger noise. This connection is further confirmed in the numerical simulations we report in Sect.~\ref{sec_experiments}.

One may wonder why Alg.~\ref{algorithm_weighted_maxent} cannot be applied to target the \MODEST problem directly. In fact, Alg.~\ref{algorithm_weighted_maxent} can be applied to any Lipschitz function (i.e., function with bounded gradient) and in the restricted set $\Lambda_\eta^{(p)}$, the function $\calL_n$ used in \FWMODEST has indeed a bounded Lipschitz constant. Nonetheless, the regret analysis in Thm.~\ref{thm.weighted.entropy} adapted from~\citep{cheung2019arvixexploration} requires evaluating the objective function at \textit{every step $t$}, which in our case means evaluating $\calL_n$ at the empirical frequency $\wt\lambda_t$. Unfortunately, $\wt\lambda_t$ is a random quantity and it may not belong to $\Lambda_\eta^{(p)}$ at each step. This justifies the ergodicity assumption and the per-episode analysis of \FWMODEST, where episodes are long enough so that the ergodicity of the MDP guarantees that each state-action pair is visited \textit{enough} and $\calL_n$ is evaluated only in the restricted set, where it is \textit{well-behaved}.



\vspace{-0.1in}
\section{NUMERICAL SIMULATIONS}
\label{sec_experiments}
\vspace{-0.1in}

We illustrate how the proposed algorithms
are able to effectively adapt to the characteristics of the environment.
We consider two domains (NoisyRiverSwim and Wheel-of-Fortune) with high level of stochasticity, i.e., \mbox{$\wb{V} := \sum_{s,a} V(s,a) / (SA)$} is large, and the transitional noise is \emph{heterogeneously spread} across the state-action space, i.e., \mbox{$\sigma(V) := \sqrt{\frac{1}{SA} \sum_{s,a} \left( V(s,a) - \wb{V} \right)^2}$} is large.
Refer to App.~\ref{app:experiments} for details and additional experiments.

\textbf{Optimal solution.} We start evaluating the error $\mathcal{E}_{\pi^\star,n}$ associated to the optimal distribution $\lambda_\star$ computed for \MaxEnt \citep{cheung2019arvixexploration}, \WeightedMaxEnt (Alg.~\ref{algorithm_weighted_maxent}) and \FWMODEST (Alg.~\ref{algorithm_fw_modest}). For each 
$(s,a)$ and algorithm $\mathfrak{A}$, we estimate the transition kernel $\widehat{p}_{\pi^\star,n}^{\mathfrak{A}}(s'|s,a)$ by generating $k = n \cdot \lambda_\star^{\mathfrak{A}}(s,a) \lor 1$ samples from $p(\cdot|s,a)$ (see App.~\ref{app:experiments} for $\lambda_\star^\mathfrak{A}$). We select $n=2 \cdot 10^6$ to simulate the asymptotic behavior of the algorithms.
As shown in Tab.~\ref{tab:optimal}, \WeightedMaxEnt and \FWMODEST recover almost identical policies, leading to very similar estimation errors.
As expected, \MaxEnt is outperformed by the proposed algorithms at the task of model estimation.

\begin{table}[t]
        \small
        \setlength{\tabcolsep}{2.8pt}
        \begin{tabular}{ccc}
                \hline
                & Wheel(5) & NoisyRiverSwim(12)\\
                \hline
                \MaxEnt & $1.0045 \cdot 10^{-3}$ & $0.4197 \cdot 10^{-2}$\\
                \WeightedMaxEnt & $0.5091 \cdot 10^{-3}$ & $0.2862 \cdot 10^{-2}$\\
                \FWMODEST &  $0.5091 \cdot 10^{-3}$& $0.2851 \cdot 10^{-2}$\\
                \hline
        \end{tabular}
        \vspace{-.1in}
        \caption{Error $\mathcal{E}_{\pi^\star,n}$ for $n = 2\cdot 10^6$. We selected $\eta=0.0001$ for \FWMODEST and $\mu=0$ for
        the others.
         }
        		\label{tab:optimal}
\end{table}
\begin{figure}[t]
        \vspace{-.1in}
        \centering
        \includegraphics[width=.22\textwidth]{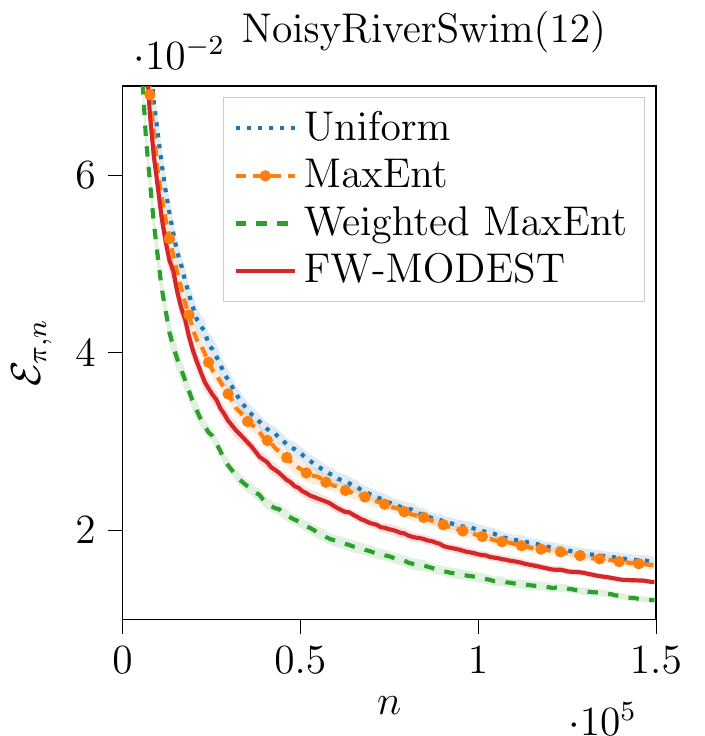}~~
        \includegraphics[width=.22\textwidth]{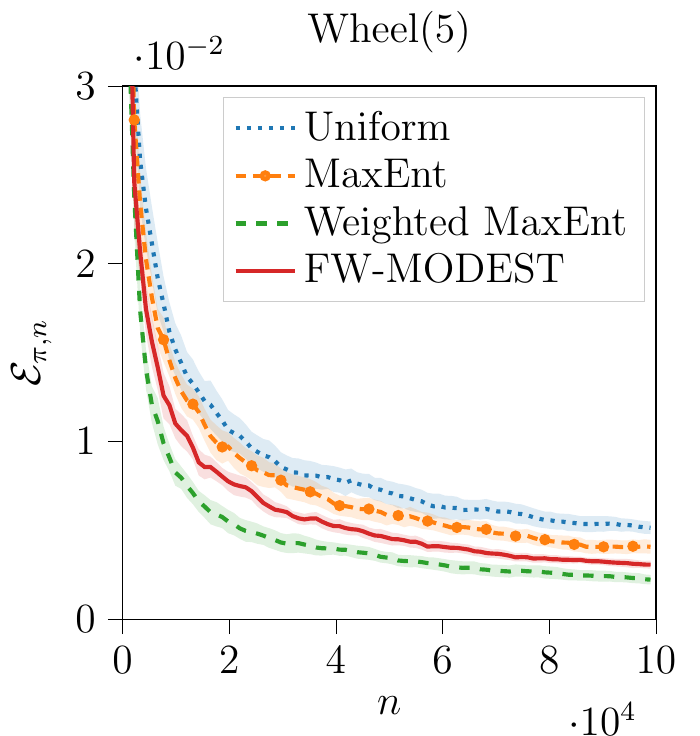}
        \vspace{-0.1in}
        \caption{Model estimation error $\mathcal{E}_{\pi,n}$ for NoisyRiverSwim(12) and Wheel(5) (averaged over $20$ runs).}
        \label{fig:riv_wheel}
\end{figure}

\textbf{Learning.} Now that we have shown that \WeightedMaxEnt and \FWMODEST have a similar asymptotic performance, we can focus on the learning process.
We first consider NoisyRiverSwim(12). Fig.~\ref{fig:riv_wheel} \emph{(left)} shows that \MaxEnt and uniform policy have a similar behavior in this domain. Both approaches are outperformed by \FWMODEST and \WeightedMaxEnt.
Despite having the same optimal solution, \WeightedMaxEnt performs better than \FWMODEST.
As suggested by the qualitative comparison between Thm.~\ref{thm:mod.est} and Thm.~\ref{thm.weighted.entropy}, the gap is due to the different learning process.
We found \WeightedMaxEnt to be more numerically stable and more sample efficient than \FWMODEST.
We observe a similar behavior in the Wheel-of-Fortune, see Fig.~\ref{fig:riv_wheel} \emph{(right)}.
In all our experiments, \WeightedMaxEnt has outperformed the other algorithms.

%

\vspace{-0.1in}
\section{CONCLUSION}
\vspace{-0.1in}

We studied the problem of reward-free exploration for model estimation and designed \FWMODEST, the first algorithm for this problem with sample complexity guarantees. We also introduced a heuristic algorithm (\WeightedMaxEnt) which requires much less restrictive assumptions and achieves better performance than \FWMODEST and \MaxEnt in our numerical simulations.



\bibliography{bibliography.bib}

\begin{thebibliography}{21}
\providecommand{\natexlab}[1]{#1}
\providecommand{\url}[1]{\texttt{#1}}
\expandafter\ifx\csname urlstyle\endcsname\relax
  \providecommand{\doi}[1]{doi: #1}\else
  \providecommand{\doi}{doi: \begingroup \urlstyle{rm}\Url}\fi

\bibitem[Araya-L{\'o}pez et~al.(2011)Araya-L{\'o}pez, Buffet, Thomas, and
  Charpillet]{araya2011active}
Mauricio Araya-L{\'o}pez, Olivier Buffet, Vincent Thomas, and Fran{\c{c}}ois
  Charpillet.
\newblock Active learning of {MDP} models.
\newblock In \emph{European Workshop on Reinforcement Learning}, pages 42--53.
  Springer, 2011.

\bibitem[Audibert et~al.(2007)Audibert, Munos, and
  Szepesv{\'a}ri]{audibert2007tuning}
Jean-Yves Audibert, R{\'e}mi Munos, and Csaba Szepesv{\'a}ri.
\newblock Tuning bandit algorithms in stochastic environments.
\newblock In \emph{International conference on algorithmic learning theory},
  pages 150--165. Springer, 2007.

\bibitem[Beck(2017)]{beck2017first}
Amir Beck.
\newblock \emph{First-order methods in optimization}, volume~25.
\newblock SIAM, 2017.

\bibitem[Berthet and Perchet(2017)]{berthet2017fast}
Quentin Berthet and Vianney Perchet.
\newblock Fast rates for bandit optimization with upper-confidence
  {F}rank-{W}olfe.
\newblock In \emph{Advances in Neural Information Processing Systems}, pages
  2225--2234, 2017.

\bibitem[Bhatnagar et~al.(2009)Bhatnagar, Sutton, Ghavamzadeh, and
  Lee]{bhatnagar2009natural}
Shalabh Bhatnagar, Richard~S Sutton, Mohammad Ghavamzadeh, and Mark Lee.
\newblock Natural actor--critic algorithms.
\newblock \emph{Automatica}, 45\penalty0 (11):\penalty0 2471--2482, 2009.

\bibitem[Cheung(2019{\natexlab{a}})]{cheung2019arvixexploration}
Wang~Chi Cheung.
\newblock Exploration-exploitation trade-off in reinforcement learning on
  online {M}arkov decision processes with global concave rewards.
\newblock \emph{arXiv preprint arXiv:1905.06466}, 2019{\natexlab{a}}.

\bibitem[Cheung(2019{\natexlab{b}})]{cheung2019regret}
Wang~Chi Cheung.
\newblock Regret minimization for reinforcement learning with vectorial
  feedback and complex objectives.
\newblock In \emph{Advances in Neural Information Processing Systems}, pages
  724--734, 2019{\natexlab{b}}.

\bibitem[Fruit et~al.(2019)Fruit, Pirotta, and
  Lazaric]{improved_analysis_UCRL2B}
Ronan Fruit, Matteo Pirotta, and Alessandro Lazaric.
\newblock Improved analysis of {UCRL2B}, 2019.
\newblock URL \url{https://rlgammazero.github.io/docs/ucrl2b_improved.pdf}.

\bibitem[Guia{\c{s}}u(1971)]{guiacsu1971weighted}
Silviu Guia{\c{s}}u.
\newblock Weighted entropy.
\newblock \emph{Reports on Mathematical Physics}, 2\penalty0 (3):\penalty0
  165--179, 1971.

\bibitem[Hazan et~al.(2019)Hazan, Kakade, Singh, and
  Van~Soest]{hazan2019provably}
Elad Hazan, Sham Kakade, Karan Singh, and Abby Van~Soest.
\newblock Provably efficient maximum entropy exploration.
\newblock In \emph{International Conference on Machine Learning}, pages
  2681--2691, 2019.

\bibitem[Jaggi(2013)]{jaggi2013revisiting}
Martin Jaggi.
\newblock Revisiting {F}rank-{W}olfe: Projection-free sparse convex
  optimization.
\newblock In \emph{International Conference on Machine Learning}, pages
  427--435, 2013.

\bibitem[Jin et~al.(2020)Jin, Krishnamurthy, Simchowitz, and Yu]{Jin20RF}
Chi Jin, Akshay Krishnamurthy, Max Simchowitz, and Tiancheng Yu.
\newblock Reward-free exploration for reinforcement learning.
\newblock In \emph{International Conference on Machine Learning}, 2020.

\bibitem[Kearns and Singh(2002)]{kearns2002near}
Michael Kearns and Satinder Singh.
\newblock Near-optimal reinforcement learning in polynomial time.
\newblock \emph{Machine learning}, 49\penalty0 (2-3):\penalty0 209--232, 2002.

\bibitem[Lim and Auer(2012)]{Lim12AE}
Shiau~Hong Lim and Peter Auer.
\newblock Autonomous exploration for navigating in {MDP}s.
\newblock In \emph{Proceedings of the 25th Annual Conference on Learning
  Theory}, 2012.

\bibitem[Maurer and Pontil(2009)]{maurer2009empirical}
Andreas Maurer and Massimiliano Pontil.
\newblock Empirical {B}ernstein bounds and sample variance penalization.
\newblock \emph{arXiv preprint arXiv:0907.3740}, 2009.

\bibitem[Paulin(2015)]{paulin2015concentration}
Daniel Paulin.
\newblock Concentration inequalities for {M}arkov chains by {M}arton couplings
  and spectral methods.
\newblock \emph{Electronic Journal of Probability}, 20, 2015.

\bibitem[Puterman(2014)]{puterman2014markov}
Martin~L Puterman.
\newblock \emph{Markov Decision Processes.: Discrete Stochastic Dynamic
  Programming}.
\newblock John Wiley \& Sons, 2014.

\bibitem[Rosenberg and Mansour(2019)]{rosenberg2019online}
Aviv Rosenberg and Yishay Mansour.
\newblock Online convex optimization in adversarial {M}arkov decision
  processes.
\newblock In \emph{International Conference on Machine Learning}, pages
  5478--5486, 2019.

\bibitem[Shekhar et~al.(2020)Shekhar, Ghavamzadeh, and
  Javidi]{shekhar2019adaptive}
Shubhanshu Shekhar, Mohammad Ghavamzadeh, and Tara Javidi.
\newblock Adaptive sampling for estimating probability distributions.
\newblock In \emph{International Conference on Machine Learning}, 2020.

\bibitem[Strehl and Littman(2008)]{strehl2008analysis}
Alexander~L Strehl and Michael~L Littman.
\newblock An analysis of model-based interval estimation for {M}arkov decision
  processes.
\newblock \emph{Journal of Computer and System Sciences}, 74\penalty0
  (8):\penalty0 1309--1331, 2008.

\bibitem[Tarbouriech and Lazaric(2019)]{tarbouriech2019active}
Jean Tarbouriech and Alessandro Lazaric.
\newblock Active exploration in {M}arkov decision processes.
\newblock In \emph{The 22nd International Conference on Artificial Intelligence
  and Statistics}, pages 974--982, 2019.

\end{thebibliography}

\newpage

\appendix


\section{Proofs in Sect.~\ref{section_exact_objective}}
\label{app_A}

\begin{proof}[Proof of Prop.~\ref{p:upper.bound}]
\label{app_problem_formulation}

By direct application of Bernstein's inequality, we have that for any triplet $(s,a,s') \in \mathcal{S} \times \mathcal{A} \times \mathcal{S}$, with probability at least $1-\delta/3$,
\begin{align*}
&\abs{p(s' \vert s,a) - \widehat{p}_{\pi,n}(s' \vert s,a)} \\ &\leq \sqrt{ \frac{2 p(s' \vert s,a)(1-p(s' \vert s,a)) \log(6 / \delta)}{T_{\pi,n}^+(s,a)} } + \frac{\log(6 / \delta)}{3 T_{\pi,n}^+(s,a)}.
\end{align*}
As a result, we can now derive a direct upper bound on both error functions. We start with $\mathcal{E}_{\pi,n}$. By a union bound argument, we get with probability at least $1-\delta/3$, for each state-action pair $(s,a) \in \mathcal{S} \times \mathcal{A}$,
\begin{align*}
\mathcal{E}_{\pi,n} &\myineeqa \sqrt{2  \log\left(\frac{6 S A n}{\delta} \right)} \frac{1}{SA} \sum_{s,a} \frac{ V(s,a) \sqrt{S}}{\sqrt{T_{\pi,n}^+(s,a)}} \\ &+ \log\left(\frac{6 S A n}{\delta} \right) \frac{1}{3SA} \sum_{s,a} \frac{S}{T_{\pi,n}^+(s,a)} \\
&\myineeqb \sqrt{2  \log\left(\frac{6 S A n}{\delta} \right)} \frac{\sqrt{2}}{SA} \sum_{s,a} \frac{ V(s,a)\sqrt{S}}{\sqrt{T_{\pi,n}(s,a) + 1}} \\ &+ \log\left(\frac{6 S A n}{\delta} \right) \frac{2}{3SA} \sum_{s,a} \frac{S}{ T_{\pi,n}(s,a) + 1} \\
&\myineeqc \frac{4}{SA} \log\left(\frac{6 SA n}{\delta} \right) \sum_{s,a}\mathcal{F}(s,a; T_{\pi,n}),
\end{align*}
where (a) uses that $T_{\pi,n}^+(s,a) \leq n$, (b) stems from the fact that $\max{ \{ x,y\} } = \frac{x+y+\abs{x-y}}{2} \geq \frac{x + y}{2}$, and (c) introduces the following function
\begin{align*}
\mathcal{F}(s,a; T) :=  \frac{ V(s,a) }{\sqrt{T(s,a) + 1}} + \frac{S}{T(s,a) + 1}.
\end{align*}
In a similar way, we have that
\begin{align*}
\mathcal{W}_{\pi,n} & \leq 4 \log\left(\frac{6 SA n}{\delta} \right) \max_{s,a} \mathcal{F}(s,a; T_{\pi,n}).
\end{align*}
~\vspace{-0.15in}
\end{proof}

\begin{proof}[Proof of Lemma~\ref{lem:hp.events}]
	From \citep{maurer2009empirical}, with probability at least $1-\delta$, for any $(s,a,s') \in \mathcal{S} \times \mathcal{A} \times \mathcal{S}$ and any time step $t \geq 1$, we have the following concentration inequality for the estimation of the standard deviation
	\begin{align*}
	&\big\vert \sqrt{\sigma_p^2(s'\vert s,a)} - \sqrt{\widehat{\sigma}^{2}_{\pi,t}(s'\vert s,a)} \big\vert \\ &\quad \quad \quad \quad \quad \quad \quad \quad \leq \sqrt{2\frac{\log\left(\frac{4 S^2 A (T^+_{\pi,t}(s,a))^2}{\delta}\right)}{T^+_{\pi,t}(s,a)}}.
	\end{align*}
Hence summing over the next states $s' \in \mathcal{S}$ gives the first statement of Lem.~\ref{lem:hp.events}. As for second statement, from Thm.~10 of \citep{improved_analysis_UCRL2B}, the fact that the confidence intervals are constructed using the empirical Bernstein inequality \citep{audibert2007tuning, maurer2009empirical} implies that, with high probability, the true transition model belongs to the confidence set $\mathcal{B}_t$ at any time $t$.
\end{proof}

\begin{proof}[Proof of Thm.~\ref{thm:mod.est}]
\label{app_extension_tarbouriech2019active_unknown_model}

\ding{172} \textbf{First}, we extend the regret bound of \citep{tarbouriech2019active} to handle the setting of an unknown transition model: for clarity, we adopt the same notation and unravel the analysis. We write $\mathcal{L} := \mathcal{L}_n$. Denote by $\rho_{k+1}$ the approximation error at the end of episode $k$, i.e., $\rho_{k+1} := \mathcal{L}(\wt{\lambda}_{k+1}) - \mathcal{L}(\lambda^{\star})$. Then from Eq.~(24) of App.~D.2 of \citep{tarbouriech2019active},
\begin{align*}
    \rho_{k+1} &= (1-\beta_k) \rho_k + C_{\eta} \beta_k^2 + \beta_k \varepsilon_{k+1} + \beta_k \Delta_{k+1},
\end{align*}
where $\beta_k := \tau_k / (t_{k+1}-1)$ and
\begin{align*}
&\varepsilon_{k+1} := \langle \nabla \mathcal{L}_n(\wt{\lambda}_k), \wh{\psi}^+_{k+1} - \psi^{\star}_{k+1} \rangle, \\
&\Delta_{k+1} := \langle \nabla \mathcal{L}_n(\wt{\lambda}_k), \wt{\psi}_{k+1} - \wh{\psi}^+_{k+1} \rangle.
\end{align*}
The tracking error $\Delta_{k+1}$ can be bounded exactly as in \citep{tarbouriech2019active}. As for the optimization error $\varepsilon_{k+1}$, we first set
\begin{align*}
    x_{k+1} := \sum_{s,a} \wh{\psi}^+_{k+1}(s,a) \frac{\alpha(t_k - 1, s,a, \delta)}{\wt{\lambda}_k(s,a)^2},
\end{align*}
where from Lem.~\ref{lem:hp.events} we have
\begin{align*}
		\alpha(t_k - 1, s,a, \delta) := \sqrt{2S \frac{\log\left(2 S^2 A \, T^2_{\pi,t}(s,a) \,\delta^{-1}\right)}{T_{\pi,t}(s,a)}}.
\end{align*}
Denote by $\mathcal{Q}$ the event under which Lem.~\ref{lem:hp.events} holds. We get under $\mathcal{Q}$ that
\begin{align*}
    \langle \nabla \mathcal{L}(\wt{\lambda}_k), \wh{\psi}^+_{k+1} \rangle &\leq \langle \nabla \wh{\mathcal{L}}^+_k(\wt{\lambda}_k), \wh{\psi}^+_{k+1} \rangle + 2 x_{k+1} \\
    &= \min_{ \widetilde{p} \in \mathcal{B}_k, \lambda \in \Lambda_{\eta}^{(\widetilde{p})}} \langle \nabla \wh{\mathcal{L}}^+_k(\wt{\lambda}_k), \lambda \rangle + 2 x_{k+1} \\
    &\myineeqa \langle \nabla \wh{\mathcal{L}}^+_k(\wt{\lambda}_k), \psi^{\star}_{k+1} \rangle + 2 x_{k+1} \\
    &\leq \langle \nabla \mathcal{L}_k(\wt{\lambda}_k), \psi^{\star}_{k+1} \rangle + 2 x_{k+1},
\end{align*}
where (a) uses that under $\mathcal{Q}$ we have $p \in \mathcal{B}_k$ (Lem.~\ref{lem:hp.events}) and that $\psi^{\star}_{k+1} \in \Lambda_{\eta}^{(p)}$. Moreover the error $x_{k+1}$ can be bounded exactly as in \citep{tarbouriech2019active}. This enables to control the optimization error $\varepsilon_{k+1}$ in a similar manner. Hence the fact that the transition dynamics are unknown does not affect the final regret bound of Thm.~1 of \citep{tarbouriech2019active}.

\ding{173} \textbf{Second}, we proceed in deriving a sample complexity result for $\mathcal{E}_{\pi,n}$. We can bound the regret $\mathcal{R}^{\mathcal{E}}_n$, which is defined as follows
\begin{align*}
    \mathcal{R}^{\mathcal{E}}_n := \mathcal{L}_n^{\mathcal{E}}(\widetilde{\lambda}_{n}) - \mathcal{L}_n^{\mathcal{E}}(\lambda^{\mathcal{E}}_{\dagger,n}),
\end{align*}
where $\lambda^{\mathcal{E}}_{\dagger,n} \in \arg\min_{\lambda \in \Lambda_{\eta}} \mathcal{L}_n^{\mathcal{E}}$. From Thm.~1 of \citep{tarbouriech2019active}, if we select episode lengths $\tau_k$ scaling as $3k^2 - 3k + 1$, we have that with overwhelming probability
\begin{align}\label{eq_analysis_regret}
    \mathcal{R}_n^{\mathcal{E}} = \wt{O} \left( \frac{\Theta^{\mathcal{E}}}{n^{1/3}} \right),
\end{align}
with $\Theta^{\mathcal{E}}$ scaling polynomially with MDP constants $S$, $A$, and $\gamma_{\min}^{-1}$, and with algorithmic dependent constants such as $\log(1/\delta)$ and $\eta^{-1}$ (see \ding{175}). Using that $\lambda_{\star}^{\mathcal{E}} \in \Lambda_{\eta}$, we have
\begin{align}
    \mathcal{L}_n^{\mathcal{E}}(\lambda_{\dagger,_n}^{\mathcal{E}}) \leq \mathcal{L}_n^{\mathcal{E}}(\lambda_{\star}^{\mathcal{E}}) \leq \mathcal{L}_{\infty}^{\mathcal{E}}(\lambda_{\star}^{\mathcal{E}}) + \frac{S}{\eta \sqrt{n}}.
\label{eq_analysis_bound}
\end{align}
Combining Eq.~\eqref{eq_analysis_regret} and \eqref{eq_analysis_bound} yields
\begin{align*}
    \mathcal{E}_{\pi,n} \leq \frac{4\log\left(\frac{6 SA n}{\delta} \right)}{\sqrt{n}} \bigg(& \mathcal{L}_{\infty}^{\mathcal{E}}(\lambda_{\star}^{\mathcal{E}}) + \frac{S}{\eta \sqrt{n}} + \wt{O} \left( \frac{\Theta^{\mathcal{E}}}{n^{1/3}} \right) \bigg).
\end{align*}
Consequently, we get
\begin{align*}
    \mathcal{E}_{\pi,n} = \wt{O} \left( \frac{\mathcal{L}^{\mathcal{E}}_{\infty}(\lambda^{\mathcal{E}}_{\star})}{\sqrt{n}} + \frac{ \Theta^{\mathcal{E}}}{n^{5/6}} \right).
\end{align*}

\ding{174} \textbf{Third}, we derive a sample complexity result for $\mathcal{W}_{\pi,n}$. The function $\overline{\mathcal{L}}_n^{\mathcal{W}}$ that is fed to the learning algorithm as objective function is both convex and smooth in $\lambda$ on the restricted simplex $\Lambda_{\eta}$ (see Prop.~\ref{prop:logsumexp.w}). Moreover, at the beginning of each episode $k$, we can compute an optimistic estimate of the unknown gradient by replacing the unknown $V(s,a)$ with the optimistic quantities $\wh{V}_{t_k}^+(s,a)$. Let us define $\lambda^{\mathcal{E}}_{\dagger,n} \in \arg\min_{\lambda \in \Lambda_{\eta}} \overline{\mathcal{L}}_n^{\mathcal{W}}$. The regret is defined as follows
\begin{align*}
    \mathcal{R}^{\mathcal{W}}_n := \overline{\mathcal{L}}_n^{\mathcal{W}}(\widetilde{\lambda}_{n}) - \overline{\mathcal{L}}_n^{\mathcal{W}}(\lambda^{\mathcal{W}}_{\dagger,n}),
\end{align*}
Then we have that
\begin{align}\label{eq_analysis_regret_2}
    \mathcal{R}^{\mathcal{W}}_n = \wt{O} \left( \frac{\Theta^{\mathcal{W}}}{n^{1/3}} \right),
\end{align}
with $\Theta^{\mathcal{W}}$ scaling polynomially with MDP constants $S$, $A$, and $\gamma_{\min}^{-1}$, and with algorithmic dependent constants such as $\log(1/\delta)$ and $\eta^{-1}$ (see \ding{175}). As a result, we have
\begin{align*}
    \mathcal{L}_n^{\mathcal{W}}(\widetilde{\lambda}_{n}) &\leq \overline{\mathcal{L}}_n^{\mathcal{W}}(\widetilde{\lambda}_{n}) \\
    &\leq \overline{\mathcal{L}}_n^{\mathcal{W}}(\lambda^{\mathcal{W}}_{\dagger,n}) + \mathcal{R}^{\mathcal{W}}_n \\
    &\leq \overline{\mathcal{L}}_n^{\mathcal{W}}(\lambda^{\mathcal{W}}_{\star}) + \mathcal{R}^{\mathcal{W}}_n \\
    &\leq \mathcal{L}_n^{\mathcal{W}}(\lambda^{\mathcal{W}}_{\star}) + \log(SA) + \mathcal{R}^{\mathcal{W}}_n \\
    &\leq \mathcal{L}^{\mathcal{W}}_{\infty}(\lambda^{\mathcal{W}}_{\star}) + \frac{S}{\eta \sqrt{n}} + \log(SA) + \mathcal{R}^{\mathcal{W}}_n,
\end{align*}
where the chain of inequalities follows from the definitions of $\lambda_{\star}^{\mathcal{W}}$ and $\lambda^{\mathcal{W}}_{\dagger,n}$ as well as point 2) of Prop.~\ref{lemma_logsumexp}. Thus,
\begin{align*}
    \mathcal{W}_{\pi,n} = \wt{O} \left( \frac{\mathcal{L}^{\mathcal{W}}_{\infty}(\lambda^{\mathcal{W}}_{\star}) + \log(SA)}{\sqrt{n}} + \frac{ \Theta^{\mathcal{W}}}{n^{5/6}} \right).
\end{align*}

\ding{175} \textbf{Finally}, we provide details on the dependencies of $\Theta^\calE$ and $\Theta^\calW$ in Eq.~\eqref{eq_analysis_regret} and \eqref{eq_analysis_regret_2}. Retracing the proof of Thm.~1 of \citep{tarbouriech2019active} enables to extract that $\Theta^\calE$ (resp.~$\Theta^\calW$) scales linearly with $S$, $A$, $\gamma_{\min}^{-1}$ and $\log(A^S/\delta)$, and polynomially with the algorithmic dependent constant $\eta^{-1}$ on which the optimization properties of $\calE$ (resp.~$\calW$) depend. Namely, we retrace that $\Theta = O\left(\alpha_{\eta} S A \gamma_{\min}^{-1} \log(A^S/\delta) + \kappa_{\eta} \right)$, where $\alpha_{\eta}$ denotes the Lipschitz constant (w.r.t.~the $\ell_{\infty}$-norm) and $\kappa_{\eta}$ denotes the smoothness constant (w.r.t.~the $\ell_{\infty}$-norm) of the considered function $\calE$ or $\calW$.
Detailing Prop.~\ref{prop:restricted.function}, on the set $\Lambda^{(p)}_\eta$ the function $\calL_n^\calE$ has a Lipschitz constant scaling as $\alpha_{\eta}^{\calE} = (1/\eta^{3/2} + 1/(\eta^{2} \sqrt{n}))/(SA)$ and a smoothness constant scaling as $\kappa_{\eta}^{\calE} = (1/\eta^{5/2} + 1/(\eta^{3}\sqrt{n}))/(SA)$. As for $\wb\calL_n^\calW$, while an exact computation of its optimization constants is more cumbersome, we can recover a Lipschitz constant scaling at most as $\alpha_{\eta}^{\calW} = 1/\eta^{4}$ and a smoothness constant scaling at most as $\kappa_{\eta}^{\calW} = 1/\eta^{5}$ (we excluded the $1/\sqrt{n}$ dependency and considered the worst-case $n=1$ for simplicity). While its smoothness constant remains polynomial in $\eta^{-1}$, the function $\wb\calL_n^\calW$ is less smooth than $\calL_n^\calE$, as explained in Sect.~\ref{ssec:alg.preparation} on the difference between Prop.~\ref{prop:restricted.function} and \ref{prop:logsumexp.w}.
\end{proof}

\begin{figure*}[t]
\centering
  \begin{tikzpicture}
\begin{scope}[local bounding box=scope1,scale=0.75]
\tikzset{VertexStyle/.style = {draw,
		shape          = circle,
		text           = black,
		inner sep      = 2pt,
		outer sep      = 0pt,
		minimum size   = 24 pt}}
\tikzset{VertexStyle2/.style = {shape          = circle,
		text           = black,
		inner sep      = 2pt,
		outer sep      = 0pt,
		minimum size   = 14 pt}}
\tikzset{Action/.style = {draw,
		shape          = circle,
		text           = black,
		fill           = black,
		inner sep      = 2pt,
		outer sep      = 0pt}}

\newcommand\xdist{3}
\newcommand\axshift{1.5}
\newcommand\ayshift{1.5}

\node[VertexStyle,fill=ForestGreen, draw=ForestGreen, text=white](s1) at (0,0) {$ 1 $};
\node[VertexStyle,fill=Blue, draw=Blue, text=white](s2) at (\xdist,0) {$ 2 $};
\node[VertexStyle,fill=ForestGreen, draw=ForestGreen, text=white](s3) at (2*\xdist,0) {$ 3 $};
\node[VertexStyle,fill=Blue, draw=Blue, text=white](s4) at (3*\xdist,0) {$ 4 $};
\node[VertexStyle,fill=ForestGreen, draw=ForestGreen, text=white](s5) at (4*\xdist,0) {$ 5 $};
\node[VertexStyle,fill=Blue, draw=Blue, text=white](s6) at (5*\xdist,0) {$ 6 $};

\begin{small}
\draw[->, >=latex,gray](s1) to [out=70,in=110,looseness=6.5] node[above]{$0.4$} (s1);
\draw[->, >=latex,gray](s1) to [out=45,in=130,looseness=1.2] node[above]{$0.6$} (s2);
\draw[->, >=latex,gray](s2) to [out=45,in=130,looseness=1.2] node[above]{$0.35$} (s3);
\draw[->, >=latex,gray](s2) to [out=70,in=110,looseness=6.5] node[above]{$0.6$} (s2);
\draw[->, >=latex,gray](s2) to [out=170,in=15,looseness=1.2] node[above]{$0.05$} (s1);
\draw[->, >=latex,gray](s3) to [out=45,in=130,looseness=1.2] node[above]{$0.35$} (s4);
\draw[->, >=latex,gray](s3) to [out=70,in=110,looseness=6.5] node[above]{$0.6$} (s3);
\draw[->, >=latex,gray](s3) to [out=170,in=15,looseness=1.2] node[above]{$0.05$} (s2);
\draw[->, >=latex,gray](s4) to [out=45,in=130,looseness=1.2] node[above]{$0.35$} (s5);
\draw[->, >=latex,gray](s4) to [out=70,in=110,looseness=6.5] node[above]{$0.6$} (s4);
\draw[->, >=latex,gray](s4) to [out=170,in=15,looseness=1.2] node[above]{$0.05$} (s3);
\draw[->, >=latex,gray](s5) to [out=45,in=130,looseness=1.2] node[above]{$0.35$} (s6);
\draw[->, >=latex,gray](s5) to [out=70,in=110,looseness=6.5] node[above]{$0.6$} (s5);
\draw[->, >=latex,gray](s5) to [out=170,in=15,looseness=1.2] node[above]{$0.05$} (s4);
\draw[->, >=latex,gray](s6) to [out=70,in=110,looseness=6.5] node[above]{$0.6$} (s6);
\draw[->, >=latex,gray](s6) to [out=170,in=15,looseness=1.2] node[above]{$0.4$} (s5);
\end{small}

\draw[dashed, ->, >=latex,black](s1) to [out=-70,in=-110,looseness=5.5] node[below]{$1$} (s1);
\draw[dashed, ->, >=latex,black](s2) to [out=225,in=-30,looseness=1.2] node[below]{$1$} (s1);
\draw[dashed, ->, >=latex,black](s3) to [out=225,in=-30,looseness=1.2] node[below]{$1$} (s2);
\draw[dashed, ->, >=latex,black](s4) to [out=225,in=-30,looseness=1.2] node[below]{$1$} (s3);
\draw[dashed, ->, >=latex,black](s5) to [out=225,in=-30,looseness=1.2] node[below]{$1$} (s4);
\draw[dashed, ->, >=latex,black](s6) to [out=225,in=-30,looseness=1.2] node[below]{$1$} (s5);

\end{scope}
\end{tikzpicture}
\caption{Example of the NoisyRiverSwim environment with $S=6$ states. The figure reports only \RIGHT (solid gray) and \LEFT (dashed black) actions. For the states that are green (i.e., odd) we have that $p(s'|s,a_{odd}) =1/S$ for all $s' \in \calS$, and $p(s|s,a_{even})=1$. For the blue states (i.e., even) we have that $p(s'|s,a_{even}) =1/S$ for all $s' \in \calS$, and $p(s|s,a_{odd})=1$.}
\label{fig:noisyriver}
\end{figure*}
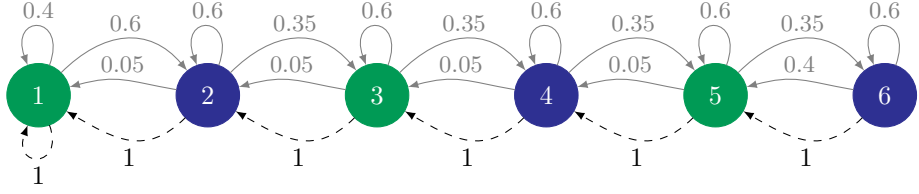

\section{Proofs in Sect.~\ref{section_approximate_objective}}

\begin{proof}[Proof of Lemma~\ref{lemma_deviation_gradient}]
	The inequalities stem from a combination of Lem.~\ref{lemma_properties_weighted_entropy} and Lem.~\ref{lem:hp.events}. More specifically, the first inequality comes from point 2.~of Lem.~\ref{lemma_properties_weighted_entropy} and from Lem.~\ref{lem:hp.events} which states that $V(s,a) \leq \wh{V}^+_t(s,a)$. As for the second inequality, we notice that
	\begin{align*}
	&\big\vert [\nabla \wh{\mathcal{H}}_t^+(\wt{\lambda}_t) - \nabla \mathcal{H}(\wt{\lambda}_t) ]_{s,a} \big\vert  \\ &\quad \quad \leq \bigg\vert \frac{\wh{V}^+_t(s,a)}{\sqrt{S L} } -  \frac{V(s,a)}{\sqrt{S L} } \bigg\vert \log\left(\frac{1}{\mu}\right) \\
	& \quad \quad = O\left( \sqrt{ \frac{\log(\frac{S A t}{\delta})}{T_{t}(s,a)}} \right).
	\end{align*}
	~\vspace{-0.1in}
\end{proof}

\begin{proof}[Proof of Lemma~\ref{lemma_properties_weighted_entropy}]
	Properties 1., 3.~and 4.~stem from Lem.~4.3 of \citep{hazan2019provably}. Moreover by choice of $\mu$, we have for any $(s,a)$ that $\mu\leq \frac{1}{e} - \lambda(s,a)$, and hence
	\begin{align*}
	\frac{1}{\log(\lambda(s,a) + \mu)} \geq 1 \geq \frac{\lambda(s,a)}{\lambda(s,a) + \mu},
	\end{align*}
	so 2.~holds since the weights are non-negative.
	~\vspace{-0.1in}
\end{proof}

\begin{proof}[Proof of Thm.~\ref{thm.weighted.entropy}]
We unravel the analysis of \citep{cheung2019arvixexploration}, whose Eq.~(11) yields
\begin{align*}
\mathcal{R}_n^{Ent} \leq&~ \frac{4 \log\left( \frac{1}{\mu}\right) \log(n)}{n} \\ &+ \frac{1}{n} \sum_{t=1}^n \nabla \mathcal{H}(\wt{\lambda}_t) ^\top (\lambda^{\star} - \mathds{1}_{s_t,a_t}),
\end{align*}
since in our setting the vectorial outcome at any state-action pair $(s,a)$ considered in \citep{cheung2019arvixexploration} corresponds to the standard basis vector for $(s,a)$ in $\mathbb{R}^{S \times A}$ (i.e., $\mathds{1}_{s,a}$). We can decompose the second term as follows
\begin{align*}
&\sum_{t=1}^n \nabla \mathcal{H}(\wt{\lambda}_t) ^\top (\lambda^{\star} - \mathds{1}_{s_t,a_t}) \\
&= \sum_{t=1}^n \nabla \wh{\mathcal{H}}_t^+(\wt{\lambda}_t) ^\top (\lambda^{\star} - \mathds{1}_{s_t,a_t}) \\
&+ \sum_{t=1}^n  \underbrace{ \left( \nabla \mathcal{H}(\wt{\lambda}_t) - \nabla \wh{\mathcal{H}}_t^+(\wt{\lambda}_t) \right) ^\top \left(\lambda^{\star} - \mathds{1}_{s_t,a_t}\right)}_{:= Z_t}.
\end{align*}
On the one hand, Lem.~\ref{lemma_deviation_gradient} yields that with high probability, $\nabla \mathcal{H}(\wt{\lambda}_t) - \nabla \wh{\mathcal{H}}_t^+(\wt{\lambda}_t) \leq 0$ component-wise. On the other hand, $\lambda^{\star}(s,a) - \mathds{1}_{s_t,a_t}$ belongs to $[-1,0]$ only at $(s,a) = (s_t,a_t)$, and is positive otherwise. Hence, we have
\begin{align*}
Z_t \leq \nabla \big[\wh{\mathcal{H}}_t^+(\wt{\lambda}_t)\big]_{s_t,a_t} - \big[\nabla \mathcal{H}(\wt{\lambda}_t)\big]_{s_t,a_t} \leq \big[u(t, \delta) \big]_{s_t,a_t},
\end{align*}
where the deviation $u(t,\delta)$ is controlled by Lem.~\ref{lemma_deviation_gradient}, which entails that
\begin{align*}
\frac{1}{n} \sum_{t=1}^n Z_t = O\left( \sqrt{\frac{\Gamma S A\log(\frac{n}{\delta})}{n}} \right).
\end{align*}
The term $\frac{1}{n} \sum_{t=1}^n \nabla \wh{\mathcal{H}}_t^+(\wt{\lambda}_t) ^\top (\lambda^{\star} - \mathds{1}_{s_t,a_t})$ can be bounded following the analysis of \citep{cheung2019arvixexploration}, since $\nabla \wh{\mathcal{H}}_{k}^+(\wt{\lambda}_k)$ corresponds exactly to the rewards fed to the \EVI procedure at the beginning of episode $k$. Moreover, from Lem.~\ref{lem:hp.events}, we have
\begin{align*}
	\max_{s,a,t} \frac{\wh{V}^+_t(s,a)}{\sqrt{S L} } \leq 2.
\end{align*}
Hence from point 3. of Lem.~\ref{lemma_properties_weighted_entropy}, we get that the Lipschitz-constant of $\wh{\mathcal{H}}_t^+$ at any time step $t$ is upper bounded by $Q := 2 \log\left(\frac{1}{\mu}\right)$. As a result, by setting the gradient threshold as equal to $Q$ in Alg.~\ref{algorithm_weighted_maxent}, we obtain the same regret bound as Eq.~(6.3) of \citep{cheung2019arvixexploration}. We finally thus get
\begin{align*}
\mathcal{R}_n^{Ent} = \wt{O}\left( \frac{D S^{1/3}}{n^{1/3}} + \frac{D \sqrt{\Gamma S A}}{\sqrt{n}} \right).
\end{align*}
\end{proof}

\section{Numerical Simulations}
\label{app:experiments}

\subsection{Description of the environments used}


\paragraph{NoisyRiverSwim.} NoisyRiverSwim is a noisy and reward-free variant of the standard environment RiverSwim introduced in~\citep{strehl2008analysis} (see Fig.~\ref{fig:noisyriver}). NoisyRiverSwim$(S)$ is a stochastic chain with $S$ states and 4 actions. At each of the $S$ states forming the chain, the two usual actions of RiverSwim are available, plus two additional ones: action $a_{even}$ (resp.~action $a_{odd}$) brings any even (resp.~odd) state to any other state of the environment. In odd states, action $a_{even}$ self-loops deterministically, and vice versa. This variant of RiverSwim injects some additional stochasticity in the transition dynamics, over which the agent has some control insofar as the behavior of the \say{random teleportation} at a given state only takes place if the correct action is executed (that is, $a_{even}$ if the state is even, $a_{odd}$ otherwise).

\paragraph{Wheel-of-Fortune.}

This environment consists of $S$ states (see Fig.~\ref{fig:wheel7}), one of which is the center of the wheel ($s_0$) and the remaining $S-1$ form the ring of the wheel ($s_1, s_2, \ldots, s_{S-1}$). $A = 5$ actions are available at each state:
\begin{itemize}[leftmargin=.2in,topsep=-2.5pt,itemsep=1pt,partopsep=0pt, parsep=0pt]
        \item In the center state $s_0$, $p(s_0 \vert s_0, a) = 1$ for all actions except for one action \SPIN $a^{\dagger}$ for which \mbox{$\forall n \in \{1, \ldots, S-1\}$}, $p(s_n \vert s_0, a^{\dagger}) = 1 / (S-1)$.
        \item For all remaining states, i.e., $\{s_1, \ldots, s_{S-1}\}$, the actions \LEFT, \RIGHT, \SELFLOOP, \CENTER  bring deterministically the agent to the left neighboring state, the right neighboring state, the same state and the center state, respectively. The action \NOISY equiprobably yields the outcome of either of the 4 previous actions.
\end{itemize}

\paragraph{Garnet.}
Finally we consider standard randomly generated environments~\citep{bhatnagar2009natural}. A Garnet instance $G(S, A, b)$ has $S$ states and $A$ actions, and for each state-action pair its branching factor $\Gamma(s,a)$ is randomly sampled in the interval $[0, b]$.

\tikzset{
  graph vertex/.style={
    circle,
    draw,
  },
  graph directed edge/.style={
    ->,
    >=stealth,
    gray,
  },
  graph tree edge/.style={
    graph directed edge
  },
  graph forward edge/.style={
    graph directed edge,
    every edge/.style={
      edge node={node [fill=white,font=\scriptsize] {f}},
      loosely dotted,
      draw,
    },
  },
  graph back edge/.style={
    graph directed edge,
    every edge/.style={
      edge node={node [fill=white,font=\scriptsize] {b}},
      densely dotted,
      draw,
    },
  },
  graph cross edge/.style={
    graph directed edge,
    every edge/.style={
      edge node={node [fill=white,font=\scriptsize] {c}},
      dotted,
      draw,
    },
  },
}

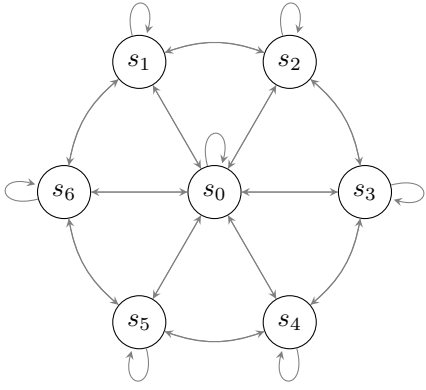
\begin{figure}[t]
\begin{tikzpicture}
\centering
  \foreach \letter [count=\c] in {A,B,C,D,E,F} {
    \node[graph vertex] (\letter) at ({-60*\c+180}:2cm) {$s_{\c}$};
  }
  \node[graph vertex] (O) at (0,0) {$s_0$};

  \path[graph tree edge]
    (A) edge [loop above] (A)
    (A) edge [bend left=20] (B)
    (B) edge [bend right=20] (A)

    (B) edge [loop above] (B)
    (B) edge [bend left=20] (C)
    (C) edge [bend right=20] (B)

    (C) edge [loop right] (C)
    (C) edge [bend left=20] (D)
    (D) edge [bend right=20] (C)

    (D) edge [loop below] (D)
    (D) edge [bend left=20] (E)
    (E) edge [bend right=20] (D)

    (E) edge [loop below] (E)
    (E) edge [bend left=20] (F)
    (F) edge [bend right=20] (E)

    (F) edge [loop left] (F)
    (F) edge [bend left=20] (A)
    (A) edge [bend right=20] (F)

    (O) edge [loop above] (O)
    (O) edge (A)
    (O) edge (B)
    (O) edge (C)
    (O) edge (D)
    (O) edge (E)
    (O) edge (F)
    (A) edge (O)
    (B) edge (O)
    (C) edge (O)
    (D) edge (O)
    (E) edge (O)
    (F) edge (O);
\end{tikzpicture}
\caption{Example of the Wheel-Of-Fortune environment  with $S=7$ states.}
\label{fig:wheel7}
\end{figure}

\subsection{Optimal solution}
In Sect.~\ref{sec_experiments}, we reported the estimation error $\mathcal{E}_{\pi^\star,n}$ computed using the optimal solution $\lambda_\star^\mathfrak{A}$ and $n = 2 \cdot 10^6$ samples. As shown in Tab.~\ref{tab:optimal}, \WeightedMaxEnt and \FWMODEST have almost identical estimation error. This can be explained by the fact that their optimal solutions are similar. Indeed, consider the Wheel(5) environment, with $\eta = 0.0001$ and $\mu=0$. The optimal solutions are

\begin{small}
\begin{align*}
        \lambda_\star^{\WeightedMaxEnt} &= \begin{pmatrix}
                0. & 0. & 0. & 0. & 0.2\\
  0. & 0. & 0. & 0. & 0.2\\
  0. & 0. & 0. & 0. & 0.2\\
  0. & 0. & 0. & 0. & 0.2\\
  0. & 0. & 0. & 0. & 0.2
        \end{pmatrix},\\
        \lambda_\star^{\FWMODEST} &= \begin{pmatrix}
        \eta & \eta &\eta &\eta&0.1999\\
        \eta & \eta &\eta &\eta&0.1995\\
        \eta & \eta &\eta &\eta&0.1995\\
        \eta & \eta &\eta &\eta&0.1995\\
        \eta & \eta &\eta &\eta&0.1995\\
        \end{pmatrix},\\
\end{align*}
\begin{align*}
        \lambda_{\star}^{\MaxEnt} &= \begin{pmatrix}
                0.043& 0.043& 0.043& 0.043& 0.1048\\
                0.043& 0.043& 0.043& 0.0176& 0.0344\\
                0.043& 0.043& 0.043& 0.0176& 0.0344\\
                0.043& 0.043& 0.043& 0.0176& 0.0344\\
                0.043& 0.043& 0.043& 0.0176& 0.0344
        \end{pmatrix}.
\end{align*}
\end{small}

It is easy to notice that \WeightedMaxEnt and \FWMODEST have almost identical solution, the only difference being due to the lower bound $\eta$.
We observed similar results in the environment NoisyRiverSwim(12).

\begin{figure*}[h]
        \centering
        \includegraphics[width=.4\textwidth]{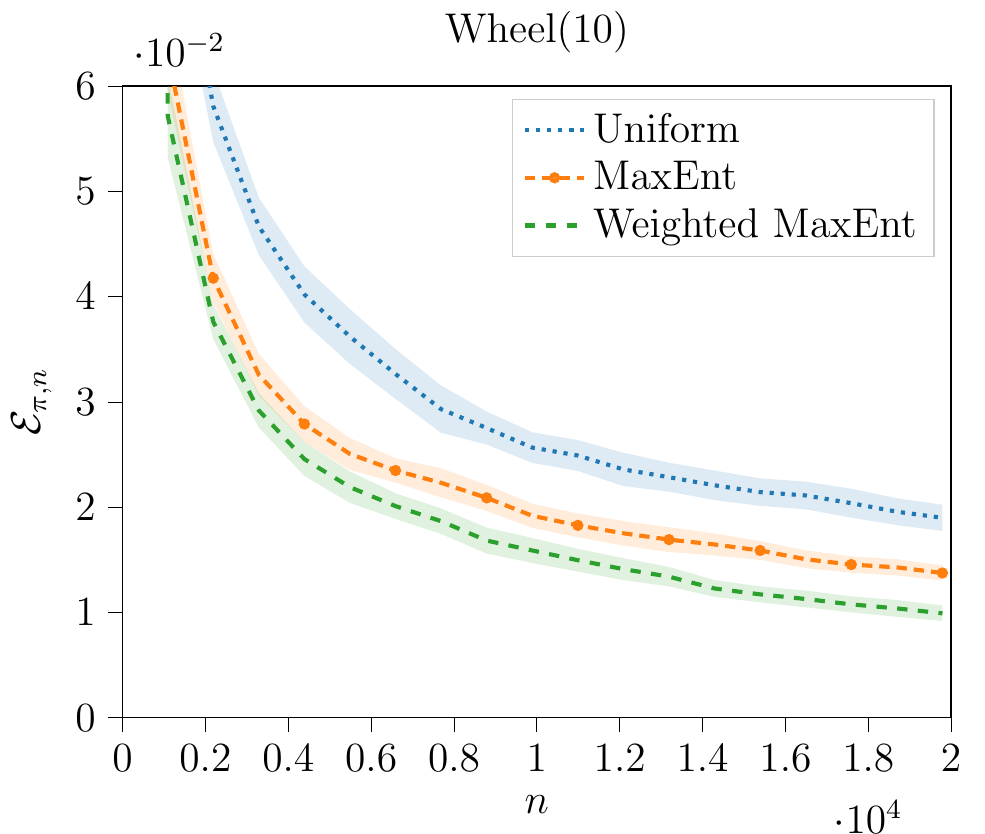} \hspace{.5in}
        \includegraphics[width=.4\textwidth]{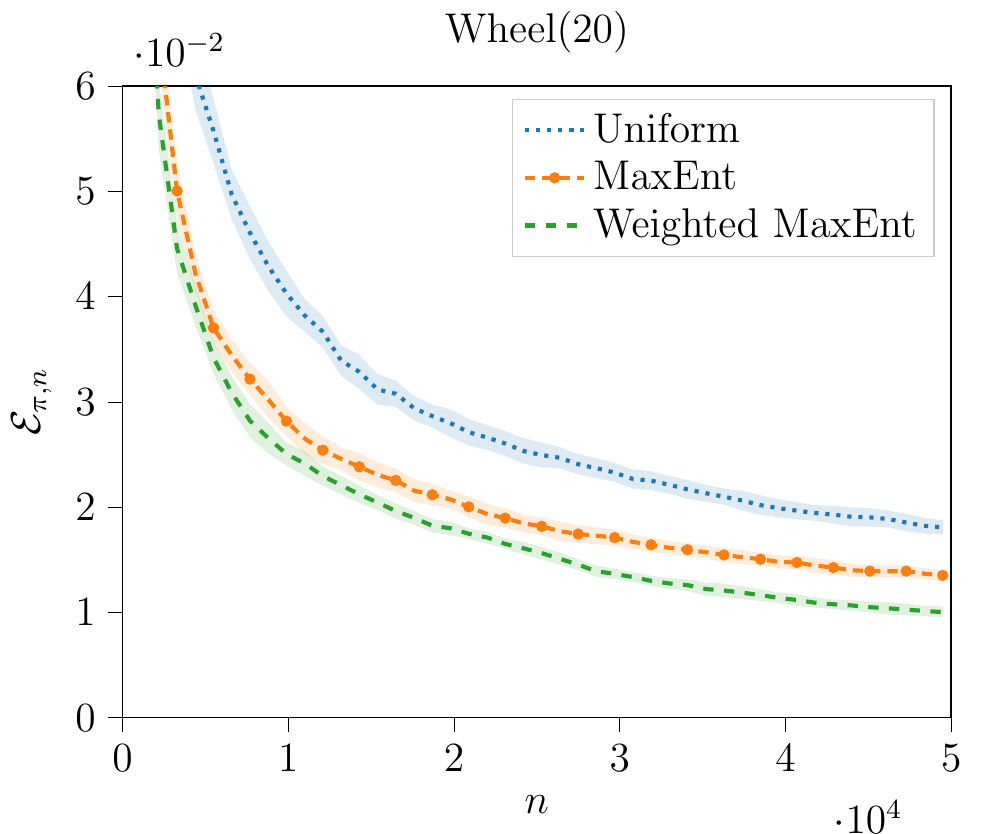}
        \caption{Estimation error $\mathcal{E}_{\pi,n}$ for the environment Wheel(10) and Wheel(20) (averaged over $20$ runs).}
        \label{fig:wheel10}
\end{figure*}

\subsection{Additional experiments}
As shown in Fig.~\ref{fig:wheel10}, when the number of states varies in the Wheel-Of-Fortune environment (respectively $10$ and $20$ states), $\WeightedMaxEnt$ outperforms \MaxEnt and the uniform policy. Finally, we compare the empirical performance of \WeightedMaxEnt, \MaxEnt and a uniform policy on a set of randomly generated Garnet instances. Fig.~\ref{fig_table} reports the average and worst-case model estimation errors $(\mathcal{E}, \mathcal{W})$ for each randomly generated environment after $n$ steps. The experimental protocol is as follows: for a fixed value of $S$, $A$ and $b$, we generated $10$ random instances of a Garnet $G(S,A,b)$ and for each instance we ran $20$ times each algorithm in order to report the mean and standard deviation of the errors. In almost all the experiments, the best performance is achieved by \WeightedMaxEnt (see Fig.~\ref{fig_table}). We observe that, as expected, the gap in performance with \MaxEnt increases with the stochasticity of the environment, which is captured by the quantity $\sigma(V)$. 

\begin{figure*}[t]
\centering

\vspace{0.2in}
\small{First table: 10 randomly generated instances of $G(5,5,5)$, with $R=20$ runs and a budget of $n=10000$:}
\vspace{0.1in}

\begin{scriptsize}

\begin{tabular}{|c|c||c|c||c|c||c|c|}
   \hline
       G(5,5,5) & $\sigma(V)$ & \multicolumn{2}{c||}{Uniform} &  \multicolumn{2}{c||}{\MaxEnt} & \multicolumn{2}{c|}{\WeightedMaxEnt} \\
\hline
0 & 0.2569 & 0.0358 $\pm$ 0.0026 & 0.1289 $\pm$ 0.02 & 0.0361 $\pm$ 0.0022 & 0.1137 $\pm$ 0.0093 & 0.032 $\pm$ 0.0018 & 0.1023 $\pm$ 0.0103 \\
\hline
1 & 0.2657 & 0.0367 $\pm$ 0.0026 & 0.1348 $\pm$ 0.0178 & 0.0348 $\pm$ 0.0015 & 0.1205 $\pm$ 0.0067 & 0.0287 $\pm$ 0.0023 & 0.1028 $\pm$ 0.0105 \\
\hline
2 & 0.2242 & 0.04 $\pm$ 0.0021 & 0.1209 $\pm$ 0.0074 & 0.0386 $\pm$ 0.0021 & 0.1182 $\pm$ 0.0134 & 0.0388 $\pm$ 0.0029 & 0.1096 $\pm$ 0.0096 \\
\hline
3 & 0.2221 & 0.0384 $\pm$ 0.0023 & 0.1266 $\pm$ 0.0175 & 0.0349 $\pm$ 0.0018 & 0.1087 $\pm$ 0.0062 & 0.0333 $\pm$ 0.0015 & 0.0994 $\pm$ 0.0103 \\
\hline
4 & 0.2264 & 0.0358 $\pm$ 0.0022 & 0.1132 $\pm$ 0.0091 & 0.0347 $\pm$ 0.0015 & 0.1085 $\pm$ 0.0103 & 0.0323 $\pm$ 0.0018 & 0.0967 $\pm$ 0.0079 \\
\hline
5 & 0.216 & 0.0397 $\pm$ 0.0023 & 0.1155 $\pm$ 0.0112 & 0.0389 $\pm$ 0.0025 & 0.1097 $\pm$ 0.0094 & 0.0389 $\pm$ 0.0019 & 0.1101 $\pm$ 0.0087 \\
\hline
6 & 0.2555 & 0.0351 $\pm$ 0.0018 & 0.1272 $\pm$ 0.0109 & 0.0347 $\pm$ 0.0019 & 0.1166 $\pm$ 0.0076 & 0.0307 $\pm$ 0.0025 & 0.0974 $\pm$ 0.0096 \\
\hline
7 & 0.2543 & 0.0299 $\pm$ 0.0019 & 0.1335 $\pm$ 0.0147 & 0.0295 $\pm$ 0.0023 & 0.1207 $\pm$ 0.0124 & 0.0251 $\pm$ 0.0018 & 0.0927 $\pm$ 0.0079 \\
\hline
8 & 0.2697 & 0.0413 $\pm$ 0.0023 & 0.1588 $\pm$ 0.0215 & 0.0367 $\pm$ 0.0018 & 0.141 $\pm$ 0.02 & 0.0312 $\pm$ 0.002 & 0.0997 $\pm$ 0.0105 \\
\hline
9 & 0.2534 & 0.0383 $\pm$ 0.0021 & 0.1142 $\pm$ 0.0101 & 0.0355 $\pm$ 0.0019 & 0.1145 $\pm$ 0.0114 & 0.0341 $\pm$ 0.0022 & 0.0936 $\pm$ 0.0076 \\
\hline
\end{tabular}

\end{scriptsize}

\vspace{0.2in}
\small{Second table: 10 randomly generated instances of $G(10,10,5)$, with $R=20$ runs and a budget of $n=20000$:}
\vspace{0.1in}

\begin{scriptsize}
\begin{tabular}{|c|c||c|c||c|c||c|c|}
   \hline
        G(10,10,5) & $\sigma(V)$ & \multicolumn{2}{c||}{Uniform} &  \multicolumn{2}{c||}{\MaxEnt} & \multicolumn{2}{c|}{\WeightedMaxEnt} \\
\hline
0 & 0.1711 & 0.0389 $\pm$ 0.0018 & 0.1417 $\pm$ 0.0098 & 0.0367 $\pm$ 0.0019 & 0.1398 $\pm$ 0.0135 & 0.034 $\pm$ 0.001 & 0.1162 $\pm$ 0.0083 \\
\hline
1 & 0.1806 & 0.0323 $\pm$ 0.0014 & 0.1266 $\pm$ 0.0082 & 0.0299 $\pm$ 0.0018 & 0.1227 $\pm$ 0.0111 & 0.0271 $\pm$ 0.0012 & 0.1057 $\pm$ 0.0066 \\
\hline
2 & 0.1591 & 0.0434 $\pm$ 0.0013 & 0.1436 $\pm$ 0.0157 & 0.0438 $\pm$ 0.0013 & 0.1386 $\pm$ 0.0087 & 0.0413 $\pm$ 0.0013 & 0.124 $\pm$ 0.0093 \\
\hline
3 & 0.1986 & 0.0333 $\pm$ 0.0015 & 0.1367 $\pm$ 0.0082 & 0.0331 $\pm$ 0.0013 & 0.1271 $\pm$ 0.013 & 0.0307 $\pm$ 0.0015 & 0.1169 $\pm$ 0.0107 \\
\hline
4 & 0.1845 & 0.0378 $\pm$ 0.0017 & 0.1351 $\pm$ 0.0125 & 0.036 $\pm$ 0.0011 & 0.1209 $\pm$ 0.0075 & 0.0332 $\pm$ 0.0011 & 0.1122 $\pm$ 0.0072 \\
\hline
5 & 0.1794 & 0.0358 $\pm$ 0.0016 & 0.1267 $\pm$ 0.0086 & 0.0352 $\pm$ 0.0012 & 0.1395 $\pm$ 0.0096 & 0.032 $\pm$ 0.0014 & 0.1107 $\pm$ 0.0075 \\
\hline
6 & 0.1857 & 0.0338 $\pm$ 0.001 & 0.1372 $\pm$ 0.0142 & 0.0326 $\pm$ 0.0015 & 0.1211 $\pm$ 0.01 & 0.0309 $\pm$ 0.0012 & 0.1193 $\pm$ 0.0101 \\
\hline
7 & 0.1705 & 0.039 $\pm$ 0.0013 & 0.1316 $\pm$ 0.0082 & 0.0386 $\pm$ 0.0014 & 0.1365 $\pm$ 0.0127 & 0.0383 $\pm$ 0.0016 & 0.1324 $\pm$ 0.0105 \\
\hline
8 & 0.1833 & 0.0318 $\pm$ 0.0012 & 0.1324 $\pm$ 0.0111 & 0.0307 $\pm$ 0.001 & 0.123 $\pm$ 0.0116 & 0.027 $\pm$ 0.0014 & 0.1102 $\pm$ 0.0107 \\
\hline
9 & 0.1992 & 0.0332 $\pm$ 0.0015 & 0.1282 $\pm$ 0.0069 & 0.0321 $\pm$ 0.0011 & 0.1266 $\pm$ 0.0098 & 0.029 $\pm$ 0.0014 & 0.1048 $\pm$ 0.0074 \\
\hline
\end{tabular}

\end{scriptsize}

\vspace{0.2in}
\small{Third table: 10 randomly generated instances of $G(20,10,5)$, with $R=20$ runs and a budget of $n=40000$:}
\vspace{0.1in}

\begin{scriptsize}

\begin{tabular}{|c|c||c|c||c|c||c|c|}
   \hline
       G(20,10,5) & $\sigma(V)$ & \multicolumn{2}{c||}{Uniform} &  \multicolumn{2}{c||}{\MaxEnt} & \multicolumn{2}{c|}{\WeightedMaxEnt} \\
\hline
0 & 0.1168 & 0.0399 $\pm$ 0.0012 & 0.1903 $\pm$ 0.0193 & 0.0363 $\pm$ 0.0012 & 0.1558 $\pm$ 0.0141 & 0.0352 $\pm$ 0.0015 & 0.143 $\pm$ 0.0118 \\
\hline
1 & 0.1403 & 0.0343 $\pm$ 0.0011 & 0.1878 $\pm$ 0.025 & 0.0311 $\pm$ 0.001 & 0.1388 $\pm$ 0.0095 & 0.0296 $\pm$ 0.0009 & 0.1346 $\pm$ 0.0102 \\
\hline
2 & 0.1294 & 0.0391 $\pm$ 0.0009 & 0.1751 $\pm$ 0.0114 & 0.0351 $\pm$ 0.0012 & 0.1431 $\pm$ 0.0084 & 0.0342 $\pm$ 0.001 & 0.1344 $\pm$ 0.0086 \\
\hline
3 & 0.1334 & 0.0377 $\pm$ 0.001 & 0.1644 $\pm$ 0.0118 & 0.0358 $\pm$ 0.0009 & 0.1293 $\pm$ 0.0067 & 0.0338 $\pm$ 0.0012 & 0.1302 $\pm$ 0.0068 \\
\hline
4 & 0.1307 & 0.0321 $\pm$ 0.0009 & 0.132 $\pm$ 0.0118 & 0.0321 $\pm$ 0.0011 & 0.1445 $\pm$ 0.0107 & 0.0301 $\pm$ 0.0012 & 0.1204 $\pm$ 0.0048 \\
\hline
5 & 0.1272 & 0.0387 $\pm$ 0.0011 & 0.1706 $\pm$ 0.0159 & 0.0373 $\pm$ 0.001 & 0.1599 $\pm$ 0.0169 & 0.0355 $\pm$ 0.0013 & 0.1339 $\pm$ 0.0065 \\
\hline
6 & 0.1257 & 0.037 $\pm$ 0.0014 & 0.1615 $\pm$ 0.0188 & 0.0354 $\pm$ 0.0009 & 0.1388 $\pm$ 0.008 & 0.0341 $\pm$ 0.0008 & 0.1387 $\pm$ 0.0077 \\
\hline
7 & 0.1261 & 0.0363 $\pm$ 0.001 & 0.1533 $\pm$ 0.011 & 0.0361 $\pm$ 0.0012 & 0.1448 $\pm$ 0.0111 & 0.0345 $\pm$ 0.001 & 0.141 $\pm$ 0.0114 \\
\hline
8 & 0.1306 & 0.0348 $\pm$ 0.0009 & 0.1407 $\pm$ 0.0079 & 0.0352 $\pm$ 0.0012 & 0.1496 $\pm$ 0.0106 & 0.034 $\pm$ 0.0011 & 0.137 $\pm$ 0.0105 \\
\hline
9 & 0.1275 & 0.0335 $\pm$ 0.0011 & 0.1514 $\pm$ 0.0101 & 0.0313 $\pm$ 0.0009 & 0.1397 $\pm$ 0.0086 & 0.0299 $\pm$ 0.0008 & 0.1313 $\pm$ 0.0098 \\
\hline
\end{tabular}

\end{scriptsize}

\caption{Experiments on Garnet environments. The first column is the index of the randomly generated Garnet $G(S,A,b)$, the second column is the standard deviation of the transitional noise of the Garnet instance (see Sect.~\ref{sec_experiments} for the definition of $\sigma(V)$). For each algorithm (uniform policy, \MaxEnt and \WeightedMaxEnt), the first subcolumn denotes the average estimation error $\mathcal{E}$ after $n$ steps, while the second subcolumn denotes the worst-case estimation error $\mathcal{W}$ after $n$ steps. The value after the $\pm$ sign denotes the standard deviation of the errors over the $R=20$ runs.}
\label{fig_table}
\end{figure*}

\end{document}